\pdfoutput=1

\documentclass[11pt]{article}

\usepackage{ACL2023}
\usepackage{custom}

\newcommand{\mymacro}[1]{{#1}}

\newcommand{\defn}[1]{\textbf{#1}}

\newcommand{\paroutline}[3][false]{%
    \ifnum\pdfstrcmp{#1}{true}=0
        #3
    \else
        [\textit{\textcolor{DiverseMagenta}{#2}}] \textcolor{AccentBlue}{#3}
    \fi
}

\newcommand{\uha}{\mymacro{\relsize{-1}{\textsf{UHA}}}\xspace}
\newcommand{\uhat}{\mymacro{\relsize{-1}{\textsf{UHAT}}}\xspace}
\newcommand{\uhats}{\mymacro{{\relsize{-1}{\textsf{UHAT}}}s}\xspace}
\newcommand{\aha}{\mymacro{\relsize{-1}{\textsf{AHA}}}\xspace}

\newcommand{\ind}[1]{\mathbbm{1} \left\{ #1 \right\}}

\newcommand{\R}{{\mymacro{ \mathbb{R}}}}

\newcommand{\alphabet}{{\mymacro{ \Sigma}}}
\newcommand{\generalAlphabet}{{\mymacro{ \Xi}}}

\newcommand{\stackalphabet}{{\mymacro{ \Gamma}}}

\newcommand{\kleene}[1]{{\mymacro{#1^*}}}
\newcommand{\kleeneplus}[1]{{\mymacro{#1^+}}}

\newcommand{\str}{{\mymacro{\boldsymbol{w}}}}

\newcommand{\strlet}{{\mymacro{ \str_{\leq\tstep}}}}

\newcommand{\strlen}{{\mymacro{T}}}
\newcommand{\strs}{{\mymacro{\boldsymbol{s}}}}
\newcommand{\strx}{{\mymacro{\boldsymbol{x}}}}
\newcommand{\stry}{{\mymacro{\boldsymbol{y}}}}
\newcommand{\strz}{{\mymacro{\boldsymbol{z}}}}
\newcommand{\stacksym}{{\mymacro{\stacksymbol{\gamma}}}}

\newcommand{\defeq}{\mathrel{\stackrel{\textnormal{\tiny def}}{=}}}

\newcommand{\NTo}[1]{{\mymacro{\left[ #1 \right]}}}

\newcommand{\set}[1]{{\mymacro{\{ #1 \}}}}

\newcommand{\orderedSubsets}[1]{{\mymacro{\mathrm{OS}{(#1)}}}}

\newcommand{\idxn}{{\mymacro{ n}}}

\newcommand{\idxi}{{\mymacro{ i}}}

\newcommand{\idxk}{{\mymacro{ k}}}

\newcommand{\nsymbols}{{\mymacro{ |\alphabet|}}}

\newcommand{\tstep}{{\mymacro{ t}}}

\newcommand{\eos}{{\mymacro{\textsc{eos}}}}

\newcommand{\scoref}{{\mymacro{\texttt{score}}}}

\newcommand{\reverse}{\mymacro{R}}

\newcommand{\automaton}{{\mymacro{ \mathcal{A}}}}
\newcommand{\wfsa}{{\mymacro{ \automaton}}}
\newcommand{\wfsar}{{\mymacro{ \automaton}}^\reverse}

\newcommand{\stateq}{{\mymacro{ q}}}
\newcommand{\statep}{{\mymacro{ p}}}

\newcommand{\states}{{\mymacro{ Q}}}

\newcommand{\trans}{{\mymacro{ \delta}}}
\newcommand{\transclosure}{{\mymacro{ \delta^*}}}

\newcommand{\final}{{\mymacro{ F}}}

\newcommand{\qinit}{{\mymacro{ q_{\iota}}}}

\newcommand{\fsatuple}{{\mymacro{ \mleft( \alphabet, \states, \qinit, \final, \trans \mright)}}}
\newcommand{\satuple}{{\mymacro{ \mleft( \alphabet, \states, \trans \mright)}}}

\newcommand{\dfatuple}{{\mymacro{ \mleft( \alphabet, \states, \qinit, \final, \trans \mright)}}}

\newcommand{\eqclass}[1]{{\mymacro{ \left[#1\right]}}}

\newcommand{\hiddDim}{{\mymacro{ D}}}

\newcommand{\proj}{{\mymacro{\texttt{proj}}}}
\newcommand{\projfunc}[1]{{\mymacro{\proj_{#1}}}}

\newcommand{\enc}{{\mymacro{\boldsymbol{h}}}}

\newcommand{\simplexFun}[1]{{\mymacro{ \boldsymbol{\Delta}}^{#1}}}

\newcommand{\negterm}[1]{{\mymacro{ {\raise.17ex\hbox{$\scriptstyle\sim$}} #1}}}

\newcommand{\ifcondition}{\textbf{if }}
\newcommand{\otherwisecondition}{\textbf{otherwise }}

\newcommand{\stacksymbol}[1]{{\mymacro{ #1 }}}

\newcommand{\ignore}[1]{}
\newcommand{\expandLater}[1]{}

\newcommand{\attn}{\mymacro{\texttt{att}}}
\newcommand{\attnFun}[1]{\mymacro{\attn}(#1)}

\newcommand{\tf}{\mymacro{\relsize{-0.25}{\textsf{T}}}\xspace}
\newcommand{\transformers}{\mymacro{\mathcal{T}}}

\newcommand{\tfencfun}{\mymacro{\enc}}

\newcommand{\tiebreak}{\mymacro{\mathcal{C}}}
\newcommand{\rightmost}{\mymacro{\mathlarger{\blacktriangleright}}}
\newcommand{\leftmost}{\mymacro{\mathlarger{\blacktriangleleft}}}
\newcommand{\futuremask}{\mymacro{\textnormal{\textsf{F}}}}
\newcommand{\pastmask}{\mymacro{\textnormal{\textsf{P}}}}

\newcommand{\hardmax}{\mymacro{\mathrm{hardmax}}}

\newcommand{\tflayer}{\mymacro{\mathrm{\mathcal{L}}}}
\newcommand{\tflayerPos}[1]{\mymacro{\tflayer}_{#1}}

\newcommand{\layerIdx}{\mymacro{\ell}}
\newcommand{\nLayers}{\mymacro{L}}

\newcommand{\embedFunc}{\mymacro{\texttt{embed}}}

\def\1{\mathbf{1}}

\def\eps{{\mymacro{ \varepsilon}}}

\def\rmH{{{\mymacro{ \mathbf{H}}}}}

\def\rmK{{{\mymacro{ \mathbf{K}}}}}

\def\rmQ{{{\mymacro{ \mathbf{Q}}}}}

\def\rmV{{{\mymacro{ \mathbf{V}}}}}

\def\vtheta{{{\mymacro{ \boldsymbol{\theta}}}}}

\def\va{{{\mymacro{ \bm{a}}}}}

\def\vs{{{\mymacro{ \bm{s}}}}}

\def\vx{{{\mymacro{ \bm{x}}}}}
\def\vy{{{\mymacro{ \bm{y}}}}}
\def\vz{{{\mymacro{ \bm{z}}}}}

\def\evs{{{\mymacro{ s}}}}

\def\gL{{{\mymacro{ \mathcal{L}}}}}

\def\gR{{{\mymacro{ \mathcal{R}}}}}

\def\sF{{{\mymacro{ \mathcal{F}}}}}

\def\sM{{{\mymacro{ \mathcal{M}}}}}
\def\sN{{{\mymacro{ \mathcal{N}}}}}

\def\sX{{{\mymacro{ \mathcal{X}}}}}

\newcommand{\N}{{\mymacro{ \mathbb{N}}}}
\newcommand{\NgtZero}{{\mymacro{ \N_{\geq 1}}}}

\DeclareMathOperator*{\argmax}{{\mymacro{ argmax}}}

\setlist[enumerate]{wide=0pt, leftmargin=0pt, listparindent=\parindent, itemsep=\parskip, parsep=0pt}

\newcommand{\bb}[1][]{\ifthenelse{\isempty{#1}}{\mymacro{\mathbf{b}}}{\mymacro{\mathbf{b}^{\text{#1}}}}}
\newcommand{\ff}[1][]{\ifthenelse{\isempty{#1}}{\mymacro{f}}{\mymacro{f_{\text{#1}}}}}

\newcommand{\W}[1][]{\ifthenelse{\isempty{#1}}{\mymacro{\mathbf{W}}}{\mymacro{\mathbf{W}^{\text{#1}}}}}

\newcommand{\stacktop}[1][]{\ifthenelse{\isempty{#1}}{\mymacro{\gamma^{\text{top}}}}{\mymacro{\gamma^{\text{top}}_{#1}}}}

\newcommand{\sym}{\mymacro{{w}}}
\newcommand{\syma}{\mymacro{{a}}}
\newcommand{\symb}{\mymacro{{b}}}

\newcommand{\recognizer}{\mymacro{\mathtt{R}}}
\newcommand{\recognizerFun}[1]{\mymacro{\recognizer}\mleft(#1\mright)}
\newcommand{\lang}{\mymacro{\mathbb{L}}}
\newcommand{\langFun}[1]{\mymacro{\lang}\mleft(#1\mright)}
\newcommand{\rlang}{\mymacro{\mathbb{L}}^\mymacro{R}}

\newcommand{\since}{\mymacro{\mathrel{\textbf{S}}}}
\newcommand{\until}{\mymacro{\mathrel{\textbf{U}}}}

\makeatletter
\newcommand*{\bigcdot}{}
\DeclareRobustCommand*{\bigcdot}{%
  \mathbin{\mathpalette\bigcdot@{}}%
}
\newcommand*{\bigcdot@scalefactor}{.5}
\newcommand*{\bigcdot@widthfactor}{1.15}
\newcommand*{\bigcdot@}[2]{%
  \sbox0{$#1\vcenter{}$}
  \sbox2{$#1\cdot\m@th$}%
  \hbox to \bigcdot@widthfactor\wd2{%
    \hfil
    \raise\ht0\hbox{%
      \scalebox{\bigcdot@scalefactor}{%
        \lower\ht0\hbox{$#1\bullet\m@th$}%
      }%
    }%
    \hfil
  }%
}
\makeatother

\usepackage{stackengine}
\newcommand\past{\mymacro{\ensurestackMath{%
  \stackengine{1pt}{\raisebox{-.8pt}{$\Diamond$}}{\raisebox{.5pt}{\scalebox{.5}[.5]{$-$}}}{O}{c}{F}{F}{L}}}}
\newcommand\future{\mymacro{\ensurestackMath{%
  \stackengine{1pt}{\raisebox{-.8pt}{$\Diamond$}}{\raisebox{.5pt}{\scalebox{.5}[.5]{$+$}}}{O}{c}{F}{F}{L}}}}

\newcommand{\TRUE}{\mymacro{\textsc{true}}}

\newcommand{\ltlAcr}{\mymacro{\textnormal{\textbf{LTL}}}}
\newcommand{\ltl}{\mymacro{\ltlAcr[\past,\future,\since,\until]}}

\newcommand{\ptl}{\mymacro{\ltlAcr[\past]}}
\newcommand{\ftl}{\mymacro{\ltlAcr[\future]}}
\newcommand{\utl}{\mymacro{\ltlAcr[\until]}}
\newcommand{\stl}{\mymacro{\ltlAcr[\since]}}

\DeclareMathOperator*{\ordEq}{\mymacro{\simeq_{\symOrd}}}

\newcommand{\tfFL}{\mymacro{\transformers^{\scaleto{\hspace{0.8pt}\futuremask}{5pt}}_{\scaleto{\leftmost}{4.5pt}}}}
\newcommand{\tfPR}{\mymacro{\transformers^{\scaleto{\hspace{0.8pt}\pastmask}{5pt}}_{\scaleto{\rightmost}{4.5pt}}}}
\newcommand{\tfFR}{\mymacro{\transformers^{\scaleto{\hspace{0.8pt}\futuremask}{5pt}}_{\scaleto{\rightmost}{4.5pt}}}}
\newcommand{\tfPL}{\mymacro{\transformers^{\scaleto{\hspace{0.8pt}\pastmask}{5pt}}_{\scaleto{\leftmost}{4.5pt}}}}

\newcommand{\brasp}{\mymacro{\textbf{B-RASP}}\xspace}
\newcommand{\braspFL}{\mymacro{\brasp{}^{\scaleto{\hspace{0.8pt}\futuremask}{5pt}}_{\scaleto{\leftmost}{4.5pt}}}}

\newcommand{\length}{\mymacro{T}}
\newcommand{\layernumber}{\mymacro{L}}
\NewDocumentCommand{\transformer}{O{} o o o }{
  \IfBlankTF{#1}{
    \IfNoValueTF{#2}{
      \IfNoValueTF{#4}{\rmH}{\rmH(#4)}
    }{
      \IfNoValueTF{#4}{\mymacro{\rmH_{#2,#3}}}{\mymacro{\rmH(#4)_{#2,#3}}}
    }
  }{
    \IfNoValueTF{#2}{
      \IfNoValueTF{#4}{\mymacro{\rmH^{(#1)}}}{\mymacro{\rmH^{(#1)}(#4)}}
    }{
      \IfNoValueTF{#4}{\mymacro{\rmH^{(#1)}_{#2,#3}}}{\mymacro{\rmH^{(#1)}(#4)_{#2,#3}}}
    }
  }
}
\NewDocumentCommand{\attention}{o}{\IfNoValueTF{#1}{\mymacro{\mathbf{A}}}{\mymacro{\mathbf{A}^{(#1)}}}}
\NewDocumentCommand{\ffn}{o}{\IfNoValueTF{#1}{\mymacro{\mathbf{F}}}{\mymacro{\mathbf{F}^{(#1)}}}}
\NewDocumentCommand{\fo}{o}{\IfNoValueTF{#1}{\mymacro{\textbf{FO}[\mathord<]}}{\mymacro{\textbf{FO}^{#1}[\mathord<]}}}
\NewDocumentCommand{\pfo}{o}{\IfNoValueTF{#1}{\mymacro{\textnormal{\textbf{PFO}}^2[\mathord<]}}{\mymacro{\textnormal{\textbf{PFO}}^2[\mathord<,#1]}}}
\NewDocumentCommand{\ffo}{o}{\IfNoValueTF{#1}{\mymacro{\textnormal{\textbf{FFO}}^2[\mathord<]}}{\mymacro{\textnormal{\textbf{FFO}}^2[\mathord<,#1]}}}

\newcommand{\fof}{\mymacro{\phi}}
\newcommand{\tlf}{\mymacro{\psi}}

\newcommand{\atom}{\mymacro{\pi}}

\NewDocumentCommand{\query}{o o }{\IfNoValueTF{#1}{\rmQ}{\rmQ_{#1,#2}}}
\NewDocumentCommand{\key}{o o }{\IfNoValueTF{#1}{\rmK}{\rmK_{#1,#2}}}
\NewDocumentCommand{\val}{o o }{\IfNoValueTF{#1}{\rmV}{\rmV_{#1,#2}}}

\newcommand{\bestsymbol}{\mymacro{\textnormal{\relsize{-1}\textsf{best}}}}
\newcommand{\actualOrder}{\mymacro{\textnormal{\relsize{-1}\textsf{order}}}}

\newcommand{\tffunc}{\mymacro{\boldsymbol{\lambda}}}
\newcommand{\tffuncFun}[1]{\mymacro{\boldsymbol{\tffunc}}_{#1}}

\newcommand{\symOrd}{\mymacro{\omega}}
\newcommand{\ordss}{\mymacro{\boldsymbol{\sX}}}
\newcommand{\ordssEl}{\mymacro{\boldsymbol{\vz}}}
\newcommand{\symOrdFun}[1]{\mymacro{\symOrd}\mleft(#1\mright)}

\newcommand{\maps}[2]{\mymacro{\text{Map}}\mleft(#1, #2\mright)}
\newcommand{\monoid}{\mymacro{\mathbb{M}}}

\newcommand{\relation}{\mymacro{ \preceq}}

\newcommand{\pofaAcr}{\mymacro{POFA}}

\newcommand{\rpofaAcr}{\mymacro{RPOFA}}

\newcommand{\mask}{\mymacro{m}}
\newcommand{\maskFun}[1]{\mymacro{\mask}\mleft(#1\mright)}
\newcommand{\unmaskedSet}{\mymacro{\sM}}
\newcommand{\unmaskedSetFun}[1]{\mymacro{\unmaskedSet}\mleft(#1\mright)}
\newcommand{\maximizers}{\mymacro{\sN}}
\newcommand{\maximizersFun}[1]{\mymacro{\maximizers}\mleft(#1\mright)}

\newcommand{\normf}{\mymacro{\texttt{norm}}}

\newcommand{\atomvec}{\mymacro{Q}}
\newcommand{\inductvec}{\mymacro{P}}
\newcommand{\outputvec}{\mymacro{Y}}

\newcommand{\scorevec}{\mymacro{s}}
\newcommand{\valuevec}{\mymacro{v}}
\newcommand{\defaultvec}{\mymacro{d}}
\newcommand{\braspprog}{\mymacro{\mathcal{P}}}

\newcommand{\halfreset}{\mymacro{\mathcal{H}}}

\newcommand{\homomorphism}{\mymacro{\phi}}

\newcommand{\fset}{\mymacro{F}}

\newcommand{\formulaAdd}[3]{\fof^{(#1)}_{#2 \gets #3}}

\title{Unique Hard Attention: A Tale of Two Sides}

\author{
Selim Jerad
~\;~\;~
Anej Svete
~\;~\;~
Jiaoda Li%
~\;~\;~Ryan Cotterell\\
\texttt{\{\href{mailto:sjerad@student.ethz.ch}{sjerad}, \href{mailto:anej.svete@inf.ethz.ch}{anej.svete},
\href{mailto:jiaoda.li@inf.ethz.ch}{jiaoda.li}, \href{mailto:ryan.cotterell@inf.ethz.ch}{ryan.cotterell}\}@ethz.ch}\\
    {%
\setlength{\fboxsep}{2.5pt}%
\setlength{\fboxrule}{2.5pt}%
\fcolorbox{white}{white}{
    \includegraphics[width=.15\linewidth]{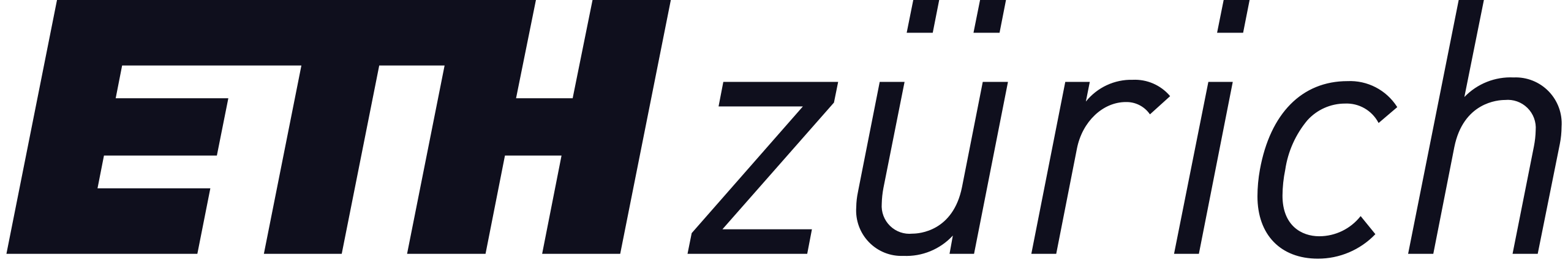}
}
}}

\begin{document}

\maketitle

\begin{quote}
    \centering
     \emph{``A wonderful fact to reflect upon, that leftmost and rightmost unique hard attention are constituted to be profoundly distinct.''}
\end{quote}

\begin{abstract}
    Understanding the expressive power of transformers has recently attracted attention, as it offers insights into their abilities and limitations.
    Many studies analyze unique hard attention transformers, where attention selects a single position that maximizes the attention scores. 
    When multiple positions achieve the maximum score, either the rightmost or the leftmost of those is chosen. 
    In this paper, we highlight the importance of this seeming triviality. 
    Recently, finite-precision transformers with both leftmost- and rightmost-hard attention were shown to be equivalent to Linear Temporal Logic (\ltlAcr{}). 
    We show that this no longer holds with only leftmost-hard attention---in that case, they correspond to a \emph{strictly weaker} fragment of \ltlAcr{}.
    Furthermore, we show that models with leftmost-hard attention are equivalent to \emph{soft} attention, suggesting they may better approximate real-world transformers than right-attention models.
    These findings refine the landscape of transformer expressivity and underscore the role of attention directionality.
\end{abstract}

\section{Introduction}

Much work has recently been done on understanding the capabilities and limitations of transformers \citep{vaswani2023attentionneed}. Collectively, the body of work on the representational capacity of transformers has provided a nuanced picture of the landscape \citep[][\textit{inter alia}]{JMLR:v22:20-302,Hahn_2020,chiang-cholak-2022-overcoming,hao2022formallanguagerecognitionhard,merrill-etal-2022-saturated,merrill2023logicexpressinglogprecisiontransformers,chiang2023tighterboundsexpressivitytransformer,yang2024maskedhardattentiontransformersrecognize,Strobl_2024,svete-cotterell-2024-transformers,nowak-etal-2024-representational,jiaodaspaper}.\footnote{\cref{app:related-work} holds a more detailed overview of related work.}
Any such investigation begins by choosing an idealization of the architecture. 
A common modeling decision is to use \defn{unique hard attention} (\uha), which selects a single position maximizing the attention scores \citep{Hahn_2020,hao2022formallanguagerecognitionhard,barcelo2023logicallanguagesacceptedtransformer}.
In one of the first exact descriptions of \uha expressivity, \citet{yang2024maskedhardattentiontransformersrecognize} show that \uha transformers (\uhats) with no positional encodings, strict future masking, and \emph{either leftmost} or \emph{rightmost} tiebreaking are equivalent to linear temporal logic, \ltlAcr{}.
This connects \uhats to well-understood formalism such as the star-free languages and counter-free automata.
An incautious reading of \citeauthor{yang2024maskedhardattentiontransformersrecognize}'s result could lead one to generalize the equivalence to all \uhats.
However, we zoom in on the overlooked choice of tiebreaking and show that it markedly impacts the model's expressivity.

We show \uhats with \emph{only leftmost} tiebreaking are strictly less expressive by relating them to a fragment of \ltlAcr{}, $\ptl$.
We do so by adapting \citeposs{yang2024maskedhardattentiontransformersrecognize} proofs: We describe a variant of the \brasp programming language defined therein, restricted to leftmost tiebreaking, and show it to be equivalent to $\ptl$.
Further, leveraging the recent results by \citet{jiaodaspaper} that characterize finite-precision future-masked \emph{soft}-attention transformers as equivalent to $\ptl$, we establish the equivalence of leftmost \uhats to standard softmax transformers.
Moreover, we give explicit $\ptl$ formulas and partially ordered finite-state automata (FSAs) describing leftmost \uhats.

\section{Transformer Idealization}
This section introduces the idealization of the transformer analyzed throughout the paper. 
\paragraph{Finite Precision.}
Implemented on modern hardware, computations performed by transformers rely on a fixed number of bits. This makes finite-precision a more realistic assumption than unbounded or growing (w.r.t. input length) precision \citep{barcelo2023logicallanguagesacceptedtransformer, Hahn_2020, hao2022formallanguagerecognitionhard, merrill2022saturatedtransformersconstantdepththreshold, merrill2023logicexpressinglogprecisiontransformers}.

\paragraph{No positional encodings.} 
We wish to study what we consider to be a \textit{barebone} idealization of the transformer, because this enables us to understand the exact expressive power of \uha. 
Moreover, finite-precision positional encodings correspond precisely to monadic predicates in \ltlAcr{} formulas \citep{yang2024maskedhardattentiontransformersrecognize}, yielding a predictable and well-understood extension to our analysis.

\paragraph{Unique hard attention.}
While the original transformer uses soft attention \citep{vaswani2023attentionneed}, theoretical work largely analyzes hard attention \citep[][\textit{inter alia}]{ merrill2022saturatedtransformersconstantdepththreshold, hao2022formallanguagerecognitionhard,yang2024maskedhardattentiontransformersrecognize, barcelo2023logicallanguagesacceptedtransformer}. 

The precise implications of this modeling decision are still unclear, but our work, combined with \citeposs{jiaodaspaper} results, reduces the gap between both models: Soft-attention transformers are equivalent to leftmost \uhats while rightmost \uhats are more expressive. 
We contextualize our findings with some more related results in \cref{tab:equivalence}.
\paragraph{Strict future masking.} 
Future masking is standard in transformer-based language models. 
We focus on \textit{strict} masking (where a position cannot attend to itself) as non-strict masking is known to reduce expressive power \citep{yang2024maskedhardattentiontransformersrecognize}.
Moreover, residual connections still allow the model to incorporate information vertically across layers.

\begin{table*}
    \footnotesize
    \centering
    \renewcommand{\arraystretch}{1.15} 
    \begin{tabular}{
        >{\centering\arraybackslash}p{1.4cm} 
        >{\centering\arraybackslash}p{1.8cm} 
        >{\centering\arraybackslash}p{1.1cm} 
        >{\centering\arraybackslash}p{1.75cm} 
        >{\centering\arraybackslash}p{1.75cm} 
        >{\centering\arraybackslash}p{1.55cm} 
        >{\centering\arraybackslash}p{3.75cm}
        }
        \toprule 
        Transformer & \textbf{LTL} & FO logic & Regex & Monoid & Automata & Note \\
        \midrule
        $\tfFR$, $\tfPL$ & $\ltl$ & $\fo$ & star-free & aperiodic & counter-free & \citet[][Thm. 5]{yang2024maskedhardattentiontransformersrecognize} \\
        \makecell{{\tiny future-masked} \\[-0.5ex] {\scriptsize \emph{soft} attention}} & $\ptl$ & $\pfo$ & $\gR$-expression & $\gR$-trivial & POFA & \citet{jiaodaspaper} \\
        $\tfFL$ & $\ptl$ & $\pfo$ & $\gR$-expression & $\gR$-trivial & POFA & \cref{thm:braspFL-to-ptl,thm:ptl-to-braspFL} \\
        $\tfPR$ & $\ftl$ & $\ffo$ & $\gL$-expression & $\gL$-trivial & RPOFA & \cref{thm:duality} \\ 
        \bottomrule
    \end{tabular}
    \caption{Known equivalences of finite-precision transformers with no positional encodings to different formalism. $\tfFR$ future-masked rightmost {\relsize{-0.75}{\textsf{UHAT}}}s. $\tfFL, \tfPL$, and $\tfPR$ are defined analogously for past masking and leftmost {\relsize{-0.75}{\textsf{UHA}}}.}
    \label{tab:equivalence}
\end{table*}

\section{The Best of {\relsize{-0.5}{\textsf{UHA}}}, The Worst of {\relsize{-0.5}{\textsf{UHA}}}} \label{sec:prologue}
This section provides a high-level overview,  intuition, roadmap, and key implications of our results.
\subsection{Separation}
We begin by building an intuition as to why, with future masking, \uha with leftmost tiebreaking $\leftmost$ is strictly less expressive than with rightmost tiebreaking $\rightmost$. 
We illustrate this in $\brasp$, a Boolean-valued programming language \citep{yang2024maskedhardattentiontransformersrecognize}, as the intermediary between \ltlAcr{} and \uhats. 
To follow the coming examples, we only need familiarity with the following attention operation:
\begin{equation}
    \textcolor{ETHPetrol}{\inductvec\mleft(\tstep\mright)} = \leftmost_{\tstep'} [\textcolor{ETHPurple}{\tstep' < \tstep}, \textcolor{ETHGreen}{\scorevec(\tstep')}] \; \textcolor{ETHRed}{\valuevec(\tstep')} : \textcolor{ETHBlue}{\defaultvec(\tstep)}
\end{equation}
$\leftmost_{\tstep'} [\textcolor{ETHPurple}{\tstep' < \tstep}, \textcolor{ETHGreen}{\scorevec(\tstep')}]$ denotes choosing (\emph{attending to}) the \emph{leftmost} ($\leftmost$) position $\tstep'$ for which $\textcolor{ETHPurple}{\tstep' < \tstep}$ holds (\underline{\relsize{-1}\textsf{F}}uture masking) and $\textcolor{ETHGreen}{\scorevec(\tstep')} = 1$.
If such a $\tstep'$ exists, $\textcolor{ETHPetrol}{\inductvec\mleft(\tstep\mright)}$ is assigned the value of the predicate $\textcolor{ETHRed}{\valuevec(\tstep')}$, and otherwise, it is assigned a default value $\textcolor{ETHBlue}{\defaultvec(\tstep)}$.
This emulates leftmost \uha, where $\textcolor{ETHPurple}{\tstep' < \tstep}$ corresponds to strict future masking, $\textcolor{ETHGreen}{\scorevec(\tstep')} = 1$ corresponds to maximizing the attention score, and $\textcolor{ETHRed}{\valuevec(\tstep')}$ corresponds to the value vector.
We define the rightmost operation $\rightmost$ analogously.

We now note two facts: 
\begin{enumerate}[label=\textit{(\roman*)}]
    \item Every leftmost operation $\leftmost$
    \begin{equation}
        \inductvec(\tstep)=\leftmost_{\tstep'} [\tstep'<\tstep, \scorevec(\tstep')]\; \valuevec(\tstep') : \defaultvec(\tstep)
    \end{equation}
    can be simulated by $\rightmost$ attentions in two steps. 
    First, we gather all positions $\tstep$ that have preceding positions $\tstep'<\tstep$ such that $\scorevec(\tstep')=1$:
    \begin{equation}
        \inductvec_1(\tstep) = \rightmost_{\tstep'} [\tstep'<\tstep,\scorevec(\tstep')] \;\scorevec(\tstep') : 0.
    \end{equation}
    Then, the leftmost position is just the single position $\tstep'$ with $\scorevec(\tstep')=1$ but $\inductvec_1(\tstep')=0$:
    \begin{equation}
        \inductvec_2(\tstep) = \rightmost_{\tstep'} [\tstep'<\tstep,\scorevec(\tstep')\land \lnot \inductvec_1(\tstep')] \; \valuevec(\tstep') : \defaultvec(\tstep).
    \end{equation}
    \item There exist operations that $\rightmost$ \emph{can} perform and $\leftmost$ \emph{cannot}.
    For instance, $\rightmost$ attention can read the value immediately to the left of the current position, i.e., $\valuevec(\tstep-1)$, as follows:
    \begin{equation}
        \inductvec(\tstep)=\rightmost_{\tstep'}[\tstep'<\tstep, 1] \; \valuevec(\tstep'):0,
    \end{equation}
    but $\leftmost$ attention cannot, as we would need $\tstep'=\tstep-1$ to be the only position with $\tstep'<\tstep$ and $\scorevec(\tstep')=1$ for all $\tstep \in[\strlen]$, which is impossible. 
\end{enumerate} 
This establishes a separation between $\braspFL$, which is limited to leftmost tiebreaking $\leftmost$ and future masking $\futuremask$, and the full $\brasp$, leading to the following: 

\begin{theorem}[Informal] \label{thm:informal-1}
    Finite-precision future-masked $\leftmost$ \uhats are \emph{weaker} than $\rightmost$ \uhats. 
\end{theorem}

\subsection{Characterizations}
In the remainder of the paper, we show that $\braspFL$ is equivalent to the fragment of \ltlAcr{} with only the $\past$ operator, denoted by $\ptl$, which in turn is equivalent to partially ordered FSAs (\pofaAcr{}). This provides an exact characterization of leftmost \uhats:
\begin{theorem}[Informal] \label{thm:informal-2}
    Finite-precision future-masked $\leftmost$ \uhats are equivalent to $\ptl$. 
\end{theorem}
In \cref{sec:tfFL-braspFL}, we formalize the theorem with proofs intentionally made analogous to \citeposs{yang2024maskedhardattentiontransformersrecognize}, in order to highlight the difference between $\leftmost$ and $\rightmost$. 
Additionally, in \cref{sec:direct-translations}, we provide alternative proofs that directly translate $\leftmost$ \uhats to $\ptl$ formulas and \pofaAcr{}.

\subsection{Implications}
Combining with \citeposs{jiaodaspaper} results that show that soft and average hard attention transformers\footnote{Average hard attention divides attention mass equally among positions maximizing the attention score.} are equivalent to $\ptl$ as well, we discover the peculiar fact that with fixed precision, soft attention, average hard attention, and leftmost \uha are all equally expressive. 

This insight could shed light on certain observed phenomena in soft-attention transformers. For instance, \citet{liu2023exposingattentionglitchesflipflop} find that transformers struggle with the flip-flop language, where the symbol following a ``read'' instruction must match the symbol following the most recent ``write'' instruction. Our results suggest that this difficulty arises because leftmost \uhats and thereby soft-attention transformers lack the ability to locate the most recent---rightmost---write instruction.

Furthermore, the fact that rightmost $\uha$ is strictly more expressive than other variants of transformers may partly explain the empirical success of positional encodings that bias attention toward recent tokens, such as ALiBi \citep{press2022train}, as they help approximate rightmost tiebreaking.

\section{Linear Temporal Logic} \label{sec:prelim}

An \defn{alphabet} $\alphabet$ is a finite, non-empty set of \defn{symbols}. 
The \defn{Kleene closure} of $\alphabet$ is $\kleene{\alphabet} = \bigcup_{n=0}^{\infty} \alphabet^n$, the set of all strings, where $\alphabet^0 \defeq \set{\eps}$ contains only the empty string. 
A \defn{language} $\lang$ is a subset of $\kleene{\alphabet}$.
We treat a \defn{language recognizer} as a function $\recognizer\colon \kleene{\alphabet} \to \{0, 1\}$ whose language is $\langFun{\recognizer} \defeq \set{\str \in \kleene{\alphabet} \mid \recognizerFun{\str} = 1}$.
Two recognizers $\recognizer_1$ and $\recognizer_2$ are \defn{equivalent} if and only if $\langFun{\recognizer_1} = \langFun{\recognizer_2}$.

\defn{Linear temporal logic} $\ltl$ is an extension of Boolean logic that considers events over time and can express time-dependent properties \citep{pnueli-LTL}. 
We define \ltlAcr{} over (finite) strings.
Formulas in $\ltl$ are composed of atomic formulas $\atom_\syma$ for $\syma\in\alphabet$, Boolean connectives $\land,\neg$,\footnote{$\lor$ can be defined using $\land$ and $\lnot$.} and four \defn{temporal operators} $\past$ (past), $\future$ (future), $\since$ (since) and $\until$ (until).
We denote by $\ptl$ the fragment with only the $\past$ operator and Boolean connectives and by $\ftl$ the fragment with only the $\future$ operator and Boolean connectives. 

Given a string $\str=\sym_1\cdots \sym_{\length}$, $\ltlAcr{}$ formulas are interpreted at some position $t\in \NTo{\length}$. 
We write $\str,t\models \tlf$ to denote that $\tlf$ is $\TRUE$ on $\str$ at position $t$. 
The semantics for $\ptl$ are:\footnote{We refer to \cref{app:preliminaries} for more details on $\ltl$ and its equivalent formalisms.}
\begin{itemize}[nosep,noitemsep,leftmargin=*]
    \item $\str,t\models \atom_\syma$ $\iff$ $\sym_t=\syma$;
    \item $\str,t\models \tlf_1 \land \tlf_2$ $\iff$ $\str,t\models \tlf_1 \land \str,t\models \tlf_2$;
    \item $\str,t\models \neg \tlf$ $\iff$ $\str,t \not\models \tlf$;
    \item $\str,t\models \past \tlf$ $\iff$ $\exists t' < t: \str, t'\models \tlf$.
    \item $\str,t\models \future \tlf$ $\iff$ $\exists t' > t: \str, t'\models \tlf$;
\end{itemize}
To define string acceptance, we denote by $\length+1$ a position outside of the string and define
\begin{equation}
    \label{eg:tl_acc}
    \str\models\tlf \Leftrightarrow \str,\length+1\models \tlf.
\end{equation} 

\section{\brasp{} and $\braspFL$}\label{sec:brasp}

We now introduce \brasp in more detail.
The input to a $\brasp$ program is a string $\str \in \kleene{\alphabet}$ with $|\str| = \strlen$. 
On such an input, a \defn{$\brasp$ program} computes a sequence of Boolean vectors of size $\strlen$, with entries indexed by $\tstep \in \NTo{\length}$ in parentheses, denoted as $\inductvec\mleft(\tstep\mright)$. 
Each $\sym \in \alphabet$ gives rise to an atomic Boolean vector $\atomvec_\sym$, defined as follows. For each $\tstep \in \NTo{\length}$:
\begin{equation}
    \atomvec_\sym(\tstep) = 1 \iff \str_\tstep = \sym
\end{equation}

To streamline notation, we denote the first $\nsymbols$ vectors of the program $\atomvec_\sym$ by $\inductvec_1, \cdots, \inductvec_{\nsymbols}$. 
The $(\idxi+1)$\textsuperscript{th} vector $\inductvec_{\idxi + 1} \defeq \inductvec'$ is computed inductively using one of the following operations:
\begin{enumerate}[nosep,noitemsep,label=\textit{(\arabic*)}]
    \item \textbf{Position-wise operation}: $\inductvec'\mleft(\tstep\mright)$ is a Boolean combination of zero or more of $\{\inductvec_{\idxi'}\mleft(\tstep\mright)\}_{\idxi' = 1}^{\idxi}$.
    \item \textbf{Attention operation}: $\inductvec'\mleft(\tstep\mright)$ can be one of:
    \begin{subequations}
        \begin{align}
            \inductvec'\mleft(\tstep\mright) & = \leftmost_{\tstep'} [\maskFun{\tstep, \tstep'}, \scorevec(\tstep, \tstep')] \; \valuevec(\tstep, \tstep') : \defaultvec(\tstep) \\
            \inductvec'\mleft(\tstep\mright) & = \rightmost_{\tstep'} [\maskFun{\tstep, \tstep'}, \scorevec(\tstep, \tstep')] \; \valuevec(\tstep, \tstep') : \defaultvec(\tstep)
        \end{align}
    \end{subequations}
    where:
    \begin{itemize}[nosep,noitemsep,leftmargin=*]
        \item The \defn{mask predicate} $\mask$ is defined as either $\maskFun{\tstep,\tstep'} \defeq \ind{\tstep' < \tstep}$ for \defn{strict future masking} ($\futuremask$), or $\maskFun{\tstep,\tstep'} \defeq \ind{\tstep' > \tstep}$ for \defn{strict past masking} ($\pastmask$).
        Notice that the inequalities are \emph{strict}, meaning the current position is excluded from attention.
        This detail has been shown to \emph{increase} expressivity compared to non-strict masking \citep{yang2024maskedhardattentiontransformersrecognize,jiaodaspaper}.\footnote{Such $\uhat$s can still access information at the current position via the residual connection; similarly, $\brasp$ programs can do so using the default predicate $\defaultvec$.}
        \item The \defn{score predicate} $\scorevec(\tstep, \tstep')$ is a Boolean combination of $\{\inductvec_{\idxi'}\mleft(\tstep\mright)\}_{\idxi' = 1}^{\idxi} \cup \{\inductvec_{\idxi'}\mleft(\tstep'\mright)\}_{\idxi' = 1}^{\idxi}$.
        \item the \defn{value predicate} $\valuevec(\tstep, \tstep')$ is defined analogously, and
        \item the \defn{default value predicate} $\defaultvec(\tstep)$ is a Boolean combination of values in $\{\inductvec_{1}\mleft(\tstep\mright), \ldots, \inductvec_{\idxi}\mleft(\tstep\mright)\}$.
    \end{itemize}
    We use $\leftmost$ to denote \defn{leftmost} tiebreaking and $\rightmost$ to denote \defn{rightmost} tiebreaking.
    For $\tstep \in \NTo{\strlen}$, define the set of valid positions as:
    \begin{equation}
        \maximizersFun{\tstep} \defeq \{\tstep' \in \NTo{\strlen} \mid \maskFun{\tstep, \tstep'} = 1 \text{ and } \scorevec(\tstep, \tstep') = 1\}.
    \end{equation}
    The unique position to attend to is then selected as:
    \begin{equation}
           \tstep^* \defeq 
        \begin{cases}
            \min \maximizersFun{\tstep} & \ifcondition \leftmost \\
            \max \maximizersFun{\tstep} & \ifcondition \rightmost
        \end{cases}. 
    \end{equation}
    Finally, the semantics of the attention operation are given by:
    \begin{equation}
        \inductvec'\mleft(\tstep\mright) \defeq 
        \begin{cases}
            \valuevec(\tstep, \tstep^*) & \ifcondition |\maximizersFun{\tstep}| > 0 \\
            \defaultvec(\tstep) & \otherwisecondition
        \end{cases}.
    \end{equation}
    Note that, by Lem. 12 and Prop. 11 of \citet{yang2024maskedhardattentiontransformersrecognize}, we can rewrite every $\braspprog$ to an equivalent program where every attention operation only uses unary scores and unary values, i.e., $\scorevec(\tstep,\tstep')$ and $\valuevec(\tstep,\tstep')$ depend only on $\tstep'$.
\end{enumerate} 

$\braspFL$ is the restricted version of $\brasp$ with only leftmost tiebreaking and future masking.

To define $\braspprog$'s language, we designate a final output vector $\outputvec$ and $\strlen$ as the output position such that $\outputvec(\strlen) = 1$ signals acceptance.

\brasp is equivalent to finite-precision future-masked rightmost \uhats, $\tfFR$ \citep[][Thms. 3 and 4]{yang2024maskedhardattentiontransformersrecognize}.
A similar claim, proved in \cref{app:proofs_brasp}, holds for $\braspFL$ and leftmost \uhats, $\tfFL$.
\begin{restatable}{theorem}{braspFLtfFLEquivalence} \label{thn:braspFL-tfFL-equivalence}
    For any \uhat in $\tfFL$, there exists an equivalent $\braspFL$ program.
    For any $\braspFL$ program, there exists an equivalent \uhat in $\tfFL$.
\end{restatable}

\section{$\braspFL$ Is Equivalent to $\ptl$} \label{sec:tfFL-braspFL}
We now formalize the equivalence of $\braspFL$ and $\ptl$. 
It rests on the following two theorems. The proofs are provided in \cref{app:proofs_brasp}. They are adapted from \citet{yang2024maskedhardattentiontransformersrecognize}, with the differing parts highlighted in \textcolor{ETHRed}{red}.
\begin{restatable}{theorem}{ptlToBraspFLTheorem} \label{thm:ptl-to-braspFL} 
For any formula $\tlf$ of $\ptl$, there is a $\braspFL$ program with a Boolean vector $\inductvec_\tlf$ such that, for any input $\str$ of length $\strlen$ and all $\tstep \in \NTo{\strlen}$, we have $\str,\tstep \models \tlf \iff \inductvec_\tlf(\tstep)=1$.
\end{restatable}
\begin{restatable}{theorem}{braspFLToPtlTheorem} \label{thm:braspFL-to-ptl} 
    For any Boolean vector $\inductvec$ of a $\braspFL$ program $\braspprog$, there is a formula $\tlf_\inductvec$ of\/ $\ptl$ such that for any input $\str$ of length $\strlen$ and all $\tstep \in \NTo{\strlen}$, we have $\inductvec(\tstep)=1 \iff \str,\tstep \models \tlf_\inductvec$.
\end{restatable}

\citet[][Thm.~15]{yang2024maskedhardattentiontransformersrecognize} establish an alternative proof of the equivalence between $\brasp$ and full $\ltl{}$ via counter-free automata, which recognize the class of star-free languages. Analogously, we demonstrate that $\braspFL$ corresponds to \pofaAcr{}s, a subclass of counter-free automata that characterize $\ptl{}$. A translation from \pofaAcr{}s to $\braspFL$ is provided in \cref{app:pofsa-to-transformer}.

\section{Direct Descriptions of $\tfFL$} \label{sec:direct-translations}
We now describe an alternative description of $\tfFL$ that \emph{directly} translates it to $\ptl$ and \pofaAcr{}s.

\subsection{Describing $\tfFL$ with $\ptl$}

In a transformer, the contextual representation at layer $\layerIdx$, $\vx_\tstep^{(\layerIdx)}$, determines a function that computes the next representation, $\vx_\tstep^{(\layerIdx+1)}$, given the unmasked symbols using the attention mechanism.
In \uhats, this function is particularly simple: It computes $\vx_\tstep^{(\layerIdx+1)}$ by \emph{selecting} the symbol with the highest attention score (as per the tiebreaking mechanism), $\vx^{(\layerIdx)}_{\tstep^*}$, and combines it with $\vx_\tstep^{(\layerIdx)}$ via the residual connection: $\vx_\tstep^{(\layerIdx+1)} = \vx_\tstep^{(\layerIdx)} + \vx^{(\layerIdx)}_{\tstep^*}$; see \cref{fig:unique-hard-attention}.
\begin{figure}
    
    \begin{tikzpicture}[
        node distance=1.2cm and 0.9cm, 
        every node/.style={align=center},
        tape/.style={rectangle, draw=ETHBlue!70, fill=ETHBlue!20, minimum height=0.65cm, minimum width=0.9cm},
        selected/.style={rectangle, draw=ETHBlue!70, fill=ETHBlue!40, minimum height=0.65cm, minimum width=0.9cm},
        rep/.style={rectangle, draw=ETHBlue!50, fill=ETHBlue!10, minimum height=0.75cm, minimum width=0.9cm},
        selected rep/.style={rectangle, draw=ETHBlue!50, fill=ETHBlue!25, minimum height=0.75cm, minimum width=0.9cm},
        head/.style={circle, draw=ETHGreen!70, fill=ETHGreen!50, text=white, minimum size=0.8cm},
        attn arrow/.style={-{Latex[length=1.75mm,width=1.25mm]}, thick, ETHGreen!95},
        comb arrow/.style={-{Latex[length=1.75mm,width=1.25mm]}, thick, ETHRed!70},
        combiner/.style={circle, draw=ETHRed!70, fill=ETHRed!50, text=white, minimum size=0.8cm}
        ]

        \node[rep] (hx1) at (0,0) {$\scriptstyle \vx_1^{(\ell)}$};
        \node[rep] (hx2) at (0.9,0) {$\scriptstyle \vx_2^{(\ell)}$};
        \node at (2,0) {$\cdots$};
        \node[rep] (hxts) at (3.1,0) {$\scriptstyle \vx_{\tstep^*}^{(\ell)}$};
        \node at (4.2,0) {$\cdots$};
        \node[rep] (hxtm1) at (5.3,0) {$\scriptstyle \vx_{\tstep-1}^{(\ell)}$};
        \node[selected rep] (hxt) at (6.2,0) {$\scriptstyle \vx_\tstep^{(\ell)}$};
        
        \node[head] (head) at (2.5,1.3) {$\scriptstyle \attn$};

        \draw[attn arrow, line width=1.2pt] (hxts.north) to[out=90,in=290] (head.south);
        \draw[attn arrow, opacity=0.2] (hxtm1.north) to[out=90,in=270] (head.south);
        \draw[attn arrow, opacity=0.2] (hx1.north) to[out=90,in=270] (head.south);
        \draw[attn arrow, opacity=0.2] (hx2.north) to[out=90,in=270] (head.south);
        
        \node[combiner] (comb) at (4.8,1.5) {$\scriptstyle +$};
        
        \draw[attn arrow] (head.east) to[out=0,in=180] (comb.west);
        \draw[comb arrow] (hxt.north) to[out=90,in=270] (comb.south);
        
        \node[selected rep] (output) at (6.2,1.5) {$\vx_\tstep^{(\layerIdx+1)}$};
        
        \draw[comb arrow] (comb.east) -- (output.west);
        
    \end{tikzpicture}
    
    \caption{Unique hard attention. $\textcolor{ETHGreen}{\vx_{\tstep^*}^{(\layerIdx)}}$ is \textcolor{ETHRed}{combined} with $\textcolor{ETHBlue}{\vx_\tstep^{(\layerIdx)}}$ to compute $\textcolor{ETHBlue}{\vx_\tstep^{(\layerIdx+1)} = \vx_\tstep^{(\layerIdx)}+ \vx_{\tstep^*}^{(\layerIdx)}}$.}
    \label{fig:unique-hard-attention}
\end{figure}
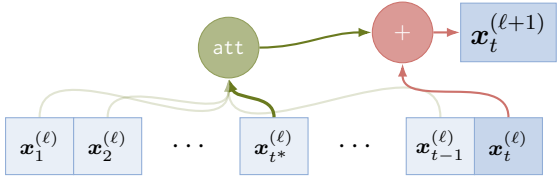

This invites the interpretation of transformer layers collecting progressively richer representations of individual symbols by selecting a new representation to append at each layer. 
We translate this idea into a set of $\ptl$ formulas of the form $\formulaAdd{\layerIdx}{\vx_\tstep^{(\layerIdx)}}{\vx_{\tstep^*}^{(\layerIdx)}}$ that keep track of the fact that the representation $\vx_\tstep^{(\layerIdx)}$ was updated with the representation $\vx_{\tstep^*}^{(\layerIdx)}$ at layer $\layerIdx$.
The full formula is presented in the proof of the following theorem in \cref{subsec:proofs_tffl_as_ptl}.
\begin{restatable}{theorem}{transformerToLTLTheorem} \label{thm:transformer-to-ltl}
    Let $\tf \in \tfFL$ be a transformer. Then, there exists an equivalent formula $\tlf \in \ptl$.
\end{restatable}

\subsection{Describing $\tfFL$ with \pofaAcr{}s}
To provide an automata-theoretic take on the result, we directly express $\tfFL$ transformers with \pofaAcr{}s in the proof of the following theorem in \cref{app:transformer-to-pofsa}.
\begin{restatable}{theorem}{transformerToPOFATheorem} 
    Let $\tf \in \tfFL$ be a transformer.
    Then, there exists an equivalent \pofaAcr{}.
\end{restatable}

\section{Discussion and Conclusion}
We establish the equivalence of future-masked finite-precision leftmost \uhats with no positional encodings, $\tfFL$, to a fragment of linear temporal logic, $\ptl$.
Together with \citeauthor{yang2024maskedhardattentiontransformersrecognize}'s and \citeauthor{jiaodaspaper}'s results, this largely completes the picture of finite-precision transformer expressivity.

\paragraph{Equivalence to soft-attention transformers.}
\Cref{sec:tfFL-braspFL} not only provides a characterization of $\tfFL$ in terms of $\ptl$, but also establishes its equivalence to future-masked, finite-precision \emph{soft}-attention transformers, in conjunction with the results of \citeposs{jiaodaspaper} (summarized in \cref{tab:equivalence}).
This equivalence yields a compelling interpretation of $\tfFL$ as a principled abstraction of soft-attention transformers---one that is more appropriate than $\tfFR$ due to the expressivity gap between soft attention and rightmost \uha. Consequently, this motivates further investigation of $\tfFL$ as a simplified yet faithful analog of more complex soft-attention architectures.

\paragraph{Dot-depth hierarchy.} 
The \defn{dot-depth hierarchy} \citep{COHEN19711} classifies star-free languages based on the minimal alternation depth of \emph{concatenation} and \emph{Boolean operations} in the regular expressions defining them. 
The hierarchy is infinite and reflects increasing expressive power. \citet{BRZOZOWSKI198032} show that the class of languages definable in $\ptl$ forms a strict subclass of dot-depth 2, while being incomparable to dot-depth 1.
In a related line of work, \citet{bhattamishra-etal-2020-ability} empirically find that transformers struggle to generalize to star-free languages with dot-depth greater than 1.

\paragraph{Until Hierarchy.}
An alternative (infinite) hierarchy spanning the star-free languages is that of \defn{until hierarchy} \citep{etessamiHierarchyTemporalLogic1996,therienTemporalLogicSemidirect2001}, which stratifies the family according to the required number of $\until$ (or equivalently, $\since$) operations in an \ltlAcr{} formula required to define a language.
Our and \citeposs{jiaodaspaper} results naturally place leftmost \uha and soft-attention transformers within the 0\textsuperscript{th} layer of this hierarchy.

\paragraph{Abilities and Limitations.}
Exact characterizations of transformers allow us to derive precise conclusions about models' abilities and limitations.
Our results, in particular, mean that $\tfFL$ transformers, like their soft-attention counterparts, \emph{cannot} model simple languages such as 
\begin{enumerate*}[label=\textit{(\roman*)},mode=unboxed]
    \item strictly local languages and $n$-gram models, which have been linked to infinite-precision \uhats \citep{svete-cotterell-2024-transformers};
    \item locally testable languages (which require detecting contiguous substrings);
    \item languages of nested parentheses of bounded depth (bounded Dyck languages) which have also been linked to infinite-precision transformers \citep{yao-etal-2021-self}; and
    \item any non-star-free languages such as \textsc{Parity}.
\end{enumerate*}
In contrast, the equivalence to $\ptl$ means that $\tfFL$ \emph{can} model simple languages such as those whose string membership depends on the presence of not necessarily contiguous subsequences or on the string prefix.
\citet{jiaodaspaper} find strong empirical evidence that this theoretical equivalence faithfully translates into the practical performance of trained transformers on formal languages.

\paragraph{A duality.}
We finally note that \emph{past} masking and \emph{rightmost} \uha has a natural characterization with $\ftl$, the fragment of \ltlAcr{} with only the $\future$ operator. 
This duality is summarized in \cref{tab:equivalence} and elaborated on in \cref{app:duality}.

\section*{Limitations}
We limit ourselves to a purely theoretical investigation of the expressivity of a particular model of a transformer.
In particular, our results hold for finite-precision transformers with leftmost hard attention and no positional encodings.
This is important to consider as the representational capacity of transformers depends on the choices made for all these components.
We do not consider learnability and training, which are opening up to be an exciting area of study and promise to bring our understanding of the representational capacity of transformers closer to what we observe in practice \citep{hahn-rofin-2024-sensitive}.
In fact, due to the discrete nature of the hard-attention mechanism, training $\uhats$ in practice is infeasible. 
Nevertheless, recent work shows how temperature scaling and unbounded positional encodings can be used to simulate hard attention with soft attention \citep{yang2024simulatinghardattentionusing}. 
We leave the empirical investigation of contrasting the performance of $\tfFL, \tfFR, \tfPR$ on various language classes to future work.

\section*{Ethical considerations}
We used AI-based tools (ChatGPT and GitHub Copilot) for writing assistance. 
We used the tools in compliance with the ACL Policy on the Use of AI Writing Assistance.

\section*{Acknowledgements}
We would like to thank Andy Yang and Michael Hahn for useful discussions and feedback on early versions of this work.
Anej Svete and Jiaoda Li are supported by the ETH AI Center Doctoral Fellowship.

\bibliography{anthology,custom}
\bibliographystyle{acl_natbib}

\appendix
\onecolumn

\section{Related Work} \label{app:related-work} 

Existing work has established a rich landscape of results on the expressivity of transformers.

\paragraph{Lower and upper bounds for \uha.}
\citet{Hahn_2020} show that \uha transformers with unbounded precision, left attention, and no masking can not recognize \textsc{Parity} (bit strings with an odd number of ones) nor \textsc{Dyck-1} (language of correctly nested parentheses of one type).
\citet{hao2022formallanguagerecognitionhard} refine this result by showing that \uha transformers with unbounded precision and left attention can recognize at most languages in \textsc{AC\textsuperscript{0}}, the family of circuits of constant depth, polynomial size, and unbounded fan-in. 
Maybe surprisingly, this suggests such transformers can not recognize even simple languages outside \textsc{AC\textsuperscript{0}} such as \textsc{Parity} and \textsc{Majority} (all bit strings in which more than half of bits are 1s). 
Other problems not in \textsc{AC\textsuperscript{0}} include sorting, integer multiplication \citep{chandra}, and integer division \citep{hessedivision}. 
\citet{barcelo2023logicallanguagesacceptedtransformer} show that \uha transformers augmented with arbitrary positional encodings are lower bounded by an extension of $\fo$ with all possible monadic numerical predicates (which includes all regular languages in \textsc{AC\textsuperscript{0}}).
\citet{yang2024maskedhardattentiontransformersrecognize} further refine the understanding of the relationship between $\fo$ and finite-precision transformers by proving the equivalence between the two when the transformers are equipped with strict future masking.

\paragraph{Average-hard attention.} 
Average-hard attention (\aha) differs from \uha in that when confronted with several positions with equal scores, they average their values to compute the next contextual representation.
It can be seen as a special case of soft-attention \citep{jiaodaspaper}. 
\citet{hao2022formallanguagerecognitionhard} show \aha \emph{unbounded} precision transformers are more expressive than \uha transformers, as they can recognize languages outside \textsc{AC\textsuperscript{0}}, such as \textsc{Parity} and \textsc{Dyck-1}.
\citet{merrill2022saturatedtransformersconstantdepththreshold} show that \aha transformers with floating-point activations can be simulated in \textsc{TC\textsuperscript{0}} (the extension of \textsc{AC\textsuperscript{0}} with majority gates, which output 1 iff at least half of the inputs are 1), while \citet{strobl2023averagehardattentiontransformersconstantdepth} extend this result by tightening the upper bound to \textsc{L-uniform TC\textsuperscript{0}} (which consists of \textsc{TC\textsuperscript{0}} circuits with the additional constraint that there exists a deterministic Turing Machine that runs in logarithmic space that can describe the circuits).

\paragraph{Transformers and logic.} 
Previous results relating transformers to logic include \citet{chiang2023tighterboundsexpressivitytransformer}, who show that finite-precision softmax transformers are upper-bounded by a generalization of first-order logic with counting quantifiers and modular arithmetic over input position indices. 
On the other hand, they show this logic to be a lower bound on the expressivity of unbounded-precision transformers.
\citet{merrill2023logicexpressinglogprecisiontransformers} contribute by characterizing a more expressive variant of transformers---they allow precision logarithmic in the input length and show an upper bound of first-order logic extended with majority-vote quantifiers. 

\paragraph{Equivalence of right-attention transformers to $\ltl$.}
\citet{yang2024maskedhardattentiontransformersrecognize} show the equivalence of $\tfFR$ and $\tfPL$ to $\ltl$. 
In their constructions, the operators $\since$ and $\until$ in $\ltl$ can only be expressed by transformers in $\tfFR$, $\tfPL$ respectively (and vice-verse, $\tfFR$ and $\tfPL$ require the operators $\since$,$\until$ respectively in their $\ltl$ formulation). Moreover, by \citet{gabbay}, the fragment of $\ltl$ consisting of only $\until$ and the Boolean connectives, denoted by $\utl$ (or, in the analogous symmetric case, $\stl$), is sufficient for equivalence $\fo$ (which is equivalent to $\ltl$ by \citet{Kamp1968-KAMTLA}).

\paragraph{Equivalence of soft-attention transformers to $\ptl$.}
\citet{jiaodaspaper} relates $\ptl$ and $\pfo$ (the fragment of $\fo$ which considers at most two distinct variables at the same time where any bound variable can only peek into the past of a free variable) to languages with $\gR$-trivial monoids \citep{BRZOZOWSKI198032}.
These are described by a set of equivalent formalisms such as $\gR$-expressions, $\gR$-trivial monoids, and partially ordered automata. 
\citet{jiaodaspaper} use this equivalence to show strict future-masked softmax transformers are equivalent to $\ptl$. 
The upper bound is proven through the equivalent $\pfo$: They show $\pfo$ can simulate a sum of a finite number of fixed-precision floating-point numbers (and thus the weighted sum in softmax), while other transformations (computations of keys and queries, dot-products, etc.) can be computed with the standard Boolean connectives. 
They show the lower bound of $\ptl$ by translating the standard Boolean operations into feedforward networks, and the $\past$ operation with the future-masked attention mechanism.
Together with \citeposs{hao2022formallanguagerecognitionhard} results, this illuminates the complexity of transformer expressivity: While in the unbounded-precision regime, \aha transformers strictly subsume \uha transformers, in the finite-precision regime, the relationship depends on the direction of attention mechanism.
The models are equivalent if leftmost tiebreaking is used and in the case of rightmost tiebreaking, \uha is strictly more expressive than \aha, the reverse of the unbounded-precision case.

\section{Background} \label{app:preliminaries}

This section provides the necessary background on the formalisms used in the paper and details some concepts introduced in the main part of the paper.

\subsection{Finite-State Automata}

\begin{definition}  \label{def:semiautomaton}
    A \defn{semiautomaton} $\wfsa$ is a 3-tuple $\satuple$ where $\alphabet$ is an alphabet, $\states$ is a finite set of \defn{states} and $\trans \colon \states \times \alphabet \rightarrow \states$ is a \defn{transition function}.
    
    We further define an \defn{initialized semiautomaton} as a semiautomaton with an initial state. 
\end{definition}

\begin{definition}  \label{def:dfa}
    A \defn{deterministic finite automaton (DFA)} $\wfsa$ is a 5-tuple $\fsatuple$ where $\satuple$ is a semiautomaton, $\qinit \in \states$ is an initial state, and $\final \subseteq \states$ is a set of final states.
\end{definition}

\begin{definition}
\label{def:p.o}
    Let $\transclosure \colon \states \times \kleene{\alphabet} \rightarrow \states$ be the transitive closure of $\trans$, defined as 
    \begin{subequations}
        \begin{align}
            \transclosure(\stateq, \sym) &= \trans(\stateq, \sym), \; \text{for } \sym \in \alphabet\\
            \transclosure(\stateq, \sym_{1}\cdots\sym_{\strlen}) &= \trans(\transclosure(\stateq,\sym_1\cdots\sym_{\strlen-1}), \sym_{\strlen})
        \end{align}
    \end{subequations}
    with $\transclosure(\stateq, \varepsilon) = \stateq$ for any $\stateq \in \states$.
    A \defn{partially ordered DFA (\pofaAcr{})} is a DFA $\wfsa = \dfatuple$ where there is a partial order relation $\relation$ on $\states$ defined as $\stateq$ $\relation$ $\statep$ if and only if $\transclosure(\stateq, \str) = \statep$ for some string $\str \in \kleene{\alphabet}$. 
\end{definition}

Intuitively, \pofaAcr{}s are \emph{acyclic} DFAs with possible self-loops, resulting in partially ordered states.

\begin{definition}  \label{def:rdfa}
    Let a DFA $\wfsa = \fsatuple$ and $\lang$ be the language it accepts. We define the \defn{reverse automaton} $\wfsar$ as the automaton that recognizes the reverse language $\rlang$ consisting of all strings in $\lang$ but reversed.
\end{definition}

\begin{definition}  \label{def:p.o.r}
    A DFA $\wfsa$ is a \defn{partially ordered reverse automaton} (\rpofaAcr) if $\wfsar$ is a \pofaAcr{}.
\end{definition}

\begin{definition}  \label{def:cascadepd}
    Let $\mymacro{ \mathcal{B}}_1 = (\alphabet, \states_1, \trans_1)$ be a semiautomaton. Let $\mymacro{ \mathcal{B}}_2 = (\states_1 \times \alphabet, \states_2, \trans_2)$ be another, possibly partial,\footnote{A partial automaton is one in which some symbol--state combinations could lead to an undefined transition.} semiautomaton such that for every $\stateq_1 \in \states_1, \stateq_2 \in \states_2$, either $\trans(\stateq_2, \langle \stateq_1, \sym \rangle)$ is defined for every $\sym \in \alphabet$ or undefined for every $\sym \in \alphabet$. The \defn{cascade product} $\mymacro{ \mathcal{B}}_1 \circ \mymacro{ \mathcal{B}}_2$ is the semiautomaton $\mymacro{ \mathcal{C}} = (\alphabet, \states, \trans)$ such that:
    \begin{itemize}
        \item $\states = \{ (\stateq_1, \stateq_2 ) : \trans_2(\stateq_2, \langle \stateq_1, \sym \rangle) \text{ is defined} \}$
        \item $\trans(\langle \stateq_1, \stateq_2 \rangle, \sym) = (\trans_1(q_1, \sym), \trans_2(\stateq_2, \langle \stateq_1, \sym \rangle))$        
    \end{itemize}
\end{definition}

\begin{definition}  \label{def:nway}
    For $n \geq 0$, an \defn{n-way fork} is an initialized semiautomaton ($\alphabet$, $\{ \stateq_0, \stateq_1, \cdots, \stateq_n \}$, $\trans$) where $\alphabet = \alphabet_0 \cup \alphabet_1 \cdots \cup \alphabet_n$, the $\alphabet_i$'s are non-empty and pairwise disjoint, $\trans(\stateq_0, \syma) = \stateq_i$ for all $\syma \in \alphabet_i$, and $\trans(\stateq_i, \syma) = \stateq_i$ for all $\syma \in \alphabet$. A \defn{half-reset} (\cref{fig:automata}) is a 1-way fork.
\end{definition}

\begin{figure}[h]
    \centering
    \begin{tikzpicture}
        \node[state, initial] (q0) [] { $\stateq_0$ }; 
        \node[state] (q1) [below = of q0] { $\stateq_1$ };  
        \draw[transition] (q0) edge[auto, bend left] node{$\alphabet_1$} (q1) 
        (q0) edge[auto, loop right] node{$\alphabet_0$} (q0)
        (q1) edge[auto, loop right] node{$\alphabet$} (q1);
    \end{tikzpicture}  
    \caption{A 1-way fork}
    \label{fig:automata}
\end{figure}
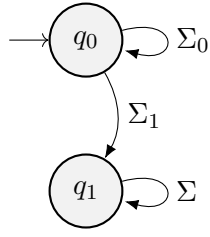

\begin{definition}  \label{def:homomorphism}
    
    A surjection $\homomorphism \colon \states \to \states'$ is an \defn{automaton homomorphism} from the semiautomata $\wfsa = \satuple$  to $\wfsa' = (\alphabet', \states', \trans')$ if for every $\stateq \in \states, \sym \in \alphabet$:
    \begin{equation}
        \homomorphism(\trans(\stateq, \sym)) = \trans'(\homomorphism(\stateq), \sym)
    \end{equation}
    In this case, we say $\wfsa'$ is homomorphic to $\wfsa$, or that $\wfsa'$ is the homomorphic image of $\wfsa$.
\end{definition}

\subsection{Syntactic Monoids}

\begin{definition}  \label{def:monoid}
    A $\defn{monoid}$ $\monoid$ is a set equipped with a binary operation and an identity element. 
\end{definition}
For instance, the \defn{free monoid} is the set $\kleene{\alphabet}$ equipped with the concatenation operation and the empty string $\eps$ as identity.

\begin{definition}  \label{def:rtrivial}
    A monoid $\monoid$ is \defn{$\gR$-trivial} if for all $\sym_1, \sym_2, \sym_3 \in \monoid$, $\sym_1\sym_2\sym_3=\sym_1$ implies $\sym_1\sym_2=\sym_1$.
    A monoid $\monoid$ is \defn{$\gL$-trivial} if for all $\sym_1, \sym_2, \sym_3 \in \monoid$, $\sym_3\sym_2\sym_1=\sym_1$ implies $\sym_2\sym_1=\sym_1$.
\end{definition}
More details about $\gR$-trivial monoids can be found in \citet{BRZOZOWSKI198032}.

\begin{definition}  \label{def:congruence}
    The \defn{syntactic congruence} $\relation_\lang$ is the equivalence relation on $\kleene{\alphabet}$ given the language $\lang$ such that for all $\strx$, $\stry \in \kleene{\alphabet}$, we have $\strx \relation_\lang \stry$ if and only if:
    \begin{equation}
        \strs \strx \strz \in \lang \iff \strs \stry \strz \in \lang \ \forall \strs, \strz \in \kleene{\alphabet}
    \end{equation}
\end{definition}

\begin{definition}\label{def:syntacticmonoid}
    The quotient monoid $\kleene{\alphabet} / \relation_\lang$ is the \defn{syntactic monoid} of $\lang$.
\end{definition}

\subsection{Regular Expressions}

A \defn{regular language} can be described by \defn{regular expressions}, which are elements of the closure of $\emptyset, \eps, $ and all $ \sym \in \alphabet$  under concatenation, concatenation, and Kleene star. 
A regular language is \defn{star-free} if the regular expression that describes the language does not require the Kleene star operator. 

\begin{definition}\label{rexpressions}
    An $\gR$-expression is a finite union of regular expressions of the form $\kleene{\alphabet_0}\sym_1\kleene{\alphabet_1}\cdots\sym_\idxn\kleene{\alphabet_\idxn}$ where $\sym_\tstep \in \alphabet$, $\alphabet_\tstep \subseteq \alphabet$ and $\sym_\tstep \notin \kleene{\alphabet_{\tstep-1}}$ for $1 \leq \tstep \leq \idxn$. An $\gL$-expression is defined analogously with the dual constraint of $\sym_\tstep \notin \kleene{\alphabet_{\tstep}}$. 
\end{definition}
For instance, $\syma\kleene{\alphabet}$ is an $\gR$-expression, while $\kleene{\alphabet}\syma$ is an $\gL$-expression.

\subsection{Linear Temporal Logic}
We now present all possible semantics in $\ltl$.
The semantics are defined inductively:
\begin{itemize}[nosep,noitemsep,leftmargin=*]
    \item $\str,t\models \atom_\syma$ $\iff$ $\sym_t=\syma$;
    \item $\str,t\models \tlf_1 \lor \tlf_2$ $\iff$ $\str,t\models \tlf_1 \lor \str,t\models \tlf_2$;
    \item $\str,t\models \tlf_1 \land \tlf_2$ $\iff$ $\str,t\models \tlf_1 \land \str,t\models \tlf_2$;
    \item $\str,t\models \neg \tlf$ $\iff$ $\str,t \not\models \tlf$;
    \item $\str,t\models \past \tlf$ $\iff$ $\exists t' < t: \str, t'\models \tlf$;
    \item $\str,t\models \future \tlf$ $\iff$ $\exists t' > t: \str, t'\models \tlf$;
    \item $\str,t\models \tlf_1 \since \tlf_2$ $\iff$ $\exists t'<t: \str,t'\models \tlf_2$ and $\str,k\models \tlf_1$ for all $k$ with $t'<k<t$;
    \item $\str,t\models \tlf_1 \until \tlf_2$ $\iff$ $\exists t'>t: \str,t'\models \tlf_2$ and $\str,k\models \tlf_1$ for all $k$ with $t<k<t'$.
\end{itemize}

$\ltl$ defines exactly the class of star-free languages \citep{Kamp1968-KAMTLA,mcnaughton1971counter}. 
$\ptl$ and $\ftl$, the fragments of \ltlAcr{} with only the $\past$ and $\future$ operators, respectively, are strictly less expressive than $\ltl$ (neither can recognize $\syma \kleene{\alphabet} \symb$, which can be recognized by $\ltl$), and have characterizations in terms of monoids, automata, and first-order logic (\cref{app:equivformal}). 

\subsection{First-Order Logic}

First-order logic with the $<$ relation, denoted by $\fo$, extends the usual propositional logic with predicates and quantifiers. 
We consider free variables $x, y, z, \cdots$ that represent positions over a string $\str$ of size $\strlen$. 
$\fo$ includes unary predicates $\atom_\sym$ for $\sym \in \alphabet$, where $\atom_\sym(x) = \TRUE$ $\iff$ there is a $\sym$ at position $x$. 
As with $\ltl$, we can inductively define formulas from $\atom_\sym$ using the standard Boolean operators, the binary predicate $<$ between positions, and the existential quantifier $\exists$. 
By \citet{Kamp1968-KAMTLA}, $\fo$ is equivalent to $\ltl$ and by \citet{mcnaughton1971counter}, it is equivalent to the star-free languages. 

$\fo[2]$ is the fragment of $\fo$ that can only consider two distinct variables at the same time.
\citet{jiaodaspaper} further define $\pfo$ as the past fragment of $\fo[2]$ that can only ``peek into the past.'' 
Namely, any single-variable formula $\fof(x)$ can only have bounded existential quantifiers of the form "$\exists y < x$". We analogously have the future fragment of $\fo[2]$, $\ffo$, where we only allow peeking into the future with quantifiers of the form "$\exists y > x$".

\subsection{Equivalence Between Formalisms}\label{app:equivformal}

Due to \citet{BRZOZOWSKI198032} and \citet{jiaodaspaper}, we have the two dual theorems:

\begin{restatable}{theorem}{}\label{thm:rtrivial}
    Let $\lang \subseteq \kleene{\alphabet}$ be a regular language, $\monoid$ be its syntactic monoid, and $\wfsa$ be the minimal DFA accepting it. The following conditions are equivalent:
    \begin{enumerate}[label=(\roman*)]
        \item $\monoid$ is $\gR$-trivial,
        \item $\lang$ can be denoted by an $\gR$-expression,
        \item $\wfsa$ is a \pofaAcr{},
        \item $\wfsa$ is the homomorphic image of a cascade product of half-resets,
        \item $\lang$ can be recognized by a formula in $\pfo$,
        \item $\lang$ can be recognized by a formula in $\ptl$.
    \end{enumerate}
\end{restatable}

\begin{restatable}{theorem}{}\label{thm:ltrivial}
    Let $\lang \subseteq \kleene{\alphabet}$ be a regular language, $\monoid$ be its syntactic monoid, and $\wfsa$ be the minimal DFA accepting it. The following conditions are equivalent:
    \begin{enumerate}[label=(\roman*)]
        \item $\monoid$ is $\gL$-trivial,
        \item $\lang$ can be denoted by an $\gL$-expression,
        \item $\wfsa$ is a \rpofaAcr{} (equivalently, $\wfsar$ is a \pofaAcr{}),
        \item $\wfsar$ is the homomorphic image of a cascade product of half-resets,
        \item $\lang$ can be recognized by a formula in $\ffo$,
        \item $\lang$ can be recognized by a formula in $\ftl$.
    \end{enumerate}
\end{restatable}

\section{Proofs of $\tfFL$ and $\ptl$ Equivalence}\label{app:proofs_brasp}
The equivalence between $\tfFL$ and $\braspFL$ follows directly from Theorems 3 and 4 of \citet{yang2024maskedhardattentiontransformersrecognize}.
\braspFLtfFLEquivalence*
\begin{proof}
    \citet[][Thms. 3 and 4]{yang2024maskedhardattentiontransformersrecognize} and the supporting lemmata (Lemmata 21 and 24) treat $\leftmost$ and $\rightmost$ separately. 
    Translations constrained to $\leftmost$ are thus special cases of their proofs.
\end{proof}

The following proofs largely follow the proofs of Lemmata 13 and 14 of \citet{yang2024maskedhardattentiontransformersrecognize}. 
We highlight the part that distinguishes $\leftmost$ and $\rightmost$ in \textcolor{ETHRed}{red}.
\ptlToBraspFLTheorem*
\begin{proof}
We proceed by induction.
\paragraph{Base case.} The atomic formulas $\atom_\syma$ can be represented by initial Boolean vectors $\atomvec_\syma$ for $\syma \in \alphabet$.
\paragraph{Inductive step.} Assume $\tlf_1$ and $\tlf_2$ can be converted to $\braspFL$ vectors $\inductvec_{\tlf_1}$ and $\inductvec_{\tlf_2}$, respectively.
We distinguish three cases of building a new formula from $\tlf_1$ and $\tlf_2$:
\begin{enumerate}[label={\textbullet~Case \textit{(\arabic*)}}:,nosep,noitemsep,wide, labelindent=0pt]
    \item $\tlf = \neg \tlf_1$.
    Add a position-wise operation:
    \begin{equation}
        \inductvec_\tlf(\tstep) = \neg P_{\tlf_1}(\tstep).
    \end{equation}
    \item $\tlf = \tlf_1 \land \tlf_2$.
    Add a position-wise operation:
    \begin{equation}
        \inductvec_\tlf(\tstep) = \inductvec_{\tlf_1}(\tstep) \land \inductvec_{\tlf_2}(\tstep).
    \end{equation}
    \item \textcolor{ETHRed}{$\tlf = \past \tlf_1$.}
    \textcolor{ETHRed}{
    Add an attention operation with future masking and $\leftmost$ tiebreaking:
    \begin{equation}
        \inductvec_{\tlf}(\tstep)=\leftmost_{\tstep'} [\tstep'<\tstep, \inductvec_{\tlf_1}(\tstep')] \; \inductvec_{\tlf_1}(\tstep') : 0.
    \end{equation}%
    }
\end{enumerate}
\end{proof}

\braspFLToPtlTheorem*
\begin{proof}
    We proceed by induction.
    \paragraph{Base case.} Each atomic vector $\atomvec_{\syma}(\tstep)$ can be translated to the atomic formula $\atom_{\syma}$.
    \paragraph{Induction step. } Assume vectors $\inductvec_1,\ldots,\inductvec_{\idxi-1}$ can be translated to $\tlf_{\inductvec_1},\ldots,\tlf_{\inductvec_{\idxi-1}}$.
    We distinguish different options of building a new vector out of $\inductvec_1,\ldots,\inductvec_{\idxi-1}$:
    \begin{enumerate}[label={\textbullet~Case \textit{(\arabic*)}}:,nosep,noitemsep,wide, labelindent=0pt]
        \item  $\inductvec_{\idxi}(\tstep)$ is a position-wise operation:
        \begin{equation}
            \inductvec_{\idxi}(\tstep) = f(\inductvec_1(\tstep), \ldots, \inductvec_{\idxi-1}(\tstep)),
        \end{equation}
        where $f$ is a Boolean function.
        We can translate $\inductvec_{\idxi}(\tstep)$ into $\tlf_{\idxi} = f(\tlf_1, \ldots, \tlf_{\idxi-1})$.
        \item \textcolor{ETHRed}{
        $\inductvec_{\idxi}(\tstep)$ is an attention operation that uses leftmost tiebreaking and future masking, that is,
        \begin{equation}
            \inductvec_{\idxi}(\tstep)=\leftmost_{\tstep'} [\tstep'<\tstep, \scorevec(\tstep')] \; \valuevec(\tstep') : \defaultvec(\tstep).
        \end{equation}%
        }
        \textcolor{ETHRed}{%
        By the inductive hypothesis, there are $\ptl$ formulas $\tlf_S$, $\tlf_V$, and $\tlf_D$ corresponding to $\scorevec$, $\valuevec$, and $\defaultvec$. 
        We can thus write $\inductvec_{\idxi}(i)$ as:
        \begin{equation}
        \label{eq:leftmost}
            \tlf_{\idxi} = (\past (\tlf_S\land \lnot \past\tlf_S \land \tlf_V)) \lor (\lnot(\past \tlf_S)\land \tlf_D).
        \end{equation}
        where $\tlf_S\land \lnot \past\tlf_S$ identifies the leftmost position that satisfies $\tlf_S$ in the string. 
        }
        \item \textcolor{ETHRed}{
        We now treat the attention operation with rightmost tiebreaking and future masking, i.e., 
        \begin{equation}
            \inductvec_{\idxi}(\tstep) = \rightmost_{\tstep'} [\tstep'<\tstep, \scorevec(\tstep')] \; \valuevec(\tstep') : \defaultvec(\tstep).
        \end{equation}
        In this case, we need to find the rightmost position satisfying $\tlf_S$ not in the \emph{entire string} $\str$, but before the \emph{current position}, which can only be realized using the $\since$ operator:
        \begin{equation}
            \tlf_{\idxi} = (\lnot\tlf_S\since (\tlf_S \land \tlf_V)) \lor (\lnot(\past \tlf_S)\land \tlf_D).
        \end{equation}
        }
    \end{enumerate}
\end{proof}

\subsection{\pofaAcr{}s as $\tfFL$} \label{app:pofsa-to-transformer} 
The Krohn--Rhodes theorem \citep{krdecomp} states that any deterministic automaton is the homomorphic image of a cascade of simple automata whose transitions induce resets or permutations of the states. 
\cref{thm:rtrivial} by \citet{BRZOZOWSKI198032} provides an analogous decomposition for \pofaAcr{}s, namely that they are homomorphic to a cascade of \emph{half-resets}. 
In a very similar manner to \citet{yang2024maskedhardattentiontransformersrecognize}'s Thm. 15, we can express this decomposition in a $\braspFL$ program, which can further be expressed by a transformer in $\tfFL$.

We first formalize how $\brasp$ can compute sequence-to-sequence functions $\kleene{\alphabet} \to \kleene{\stackalphabet}$ if for instance, we want a $\brasp$ program to simulate an automaton (i.e., output the same states traversed by the automaton when ran on the same input). We designate a set of output vectors $\outputvec_\stacksym$, for all $\stacksym \in \stackalphabet$, and the sequence-to-sequence function then maps $\tstep$ to $\stacksym$ if $\outputvec_\stacksym(\tstep)$ is true.

We then say a $\braspFL$ program $\braspprog$ simulates an initialized semiautomaton $\wfsa = \satuple$ with initial state $\stateq_\vs$ iff for every input string $\str$, the output vectors of $\braspprog$ encode the sequence of states traversed by $\wfsa$ when ran on $\str$ with initial state $\stateq_\vs$.
\begin{lemma} \label{lem:single-half-resets}
    Let $\halfreset = (\alphabet_0 \cup \alphabet_1, \{\stateq_0, \stateq_1\}, \stateq_0, \trans)$ be a half-reset. Then there exists a $\brasp$ program $\braspprog_\halfreset$ that simulates $\halfreset$ started in start state $\stateq_0$.
\end{lemma}
\begin{proof}
    We first define two $\brasp$ programs $\mathcal{B}_{0}$, $\mathcal{B}_{1}$, which return 1 if and only if the half-reset is in state $\stateq_0$ or $\stateq_1$, respectively. Namely: 
    \begin{equation}
        \begin{aligned}
            \mathcal{B}_{0}(\tstep) & = \leftmost_{\tstep'} [\tstep' < \tstep, \bigvee_{\sym \in \alphabet_1} \states_{\sym}(\tstep')] \; 0 : 1 \\
            \mathcal{B}_{1}(\tstep) & = 1 - \mathcal{B}_{0}(\tstep)
        \end{aligned}
    \end{equation}
    In half-resets, the presence or absence at any instance of symbols in $\alphabet_1$ determines whether we are in $\stateq_0$ or $\stateq_1$. 
    It suffices to check if at the current index $\tstep$, there has been a past $\sym \in \alphabet_1$, which can be done by computing $\bigvee_{\sym \in \alphabet_1} \states_{\sym}(\tstep')$ for all indices $\tstep'$ satisfying future masking.
\end{proof}

\begin{lemma} \label{lem:two-half-resets}
    Let $\wfsa = (\alphabet, \states_1, \trans_1)$ be a semiautomaton that can be simulated from state $\vs_1$ by a $\brasp$ program $\braspprog_\wfsa$. Let $\halfreset = (\states_1 \times \alphabet, \states_2, \trans_2)$ be a half-reset and let $\mymacro{ \mathcal{C}} = \wfsa \circ \halfreset$. Then, there exists a program $\braspprog_\mymacro{ \mathcal{C}}$ that simulates $\mymacro{ \mathcal{C}}$ started in state ($\vs_1, \vs_2$) for an arbitrary $\vs_2 \in \states_2$. 
\end{lemma}
\begin{proof}
    We now use for all $\stateq \in \states_1$, the predicates $\mathcal{B}_{\wfsa, \stateq}$ that denote whether $\wfsa$ is in some state $\stateq \in \states_1$ at time $\tstep$ when started at $\vs_1$. 
    By the assumption that $\wfsa$ can be simulated by a $\braspFL$ program, we have access to such predicates.
    
    We now define predicates for $\stateq \in \states_1$, $\sym \in \alphabet$:
    \begin{equation}
        \begin{aligned}
            \states'_{(\stateq, \sym)}(\tstep) & = \mathcal{B}_{\wfsa, \stateq}(\tstep) \wedge \states_{\sym}(\tstep)
        \end{aligned}
    \end{equation}
    These formulas encode the presence of an element in $\states_1 \times \alphabet$ (a state in the first semiautomaton $\wfsa$ and an alphabet symbol) for the half-reset $\halfreset$.
    
    As $\halfreset$ is a half-reset, we can classify every tuple $(\stateq, \sym) \in \states_1 \times \alphabet$ into one of two sets $\alphabet_0$ or $\alphabet_1$ (depending on if they reset to $\stateq_0$ or $\stateq_1$ in $\halfreset$). 
    As in \cref{lem:single-half-resets}, we can define predicates $\mathcal{B}_{\halfreset, \stateq}(\tstep)$ using the predicates $\states'_{(\stateq, \sym)}$ and the known classification of elements in $\states_1 \times \alphabet$ into some state $\stateq_0$ or $\stateq_1$.
    
    To finally simulate the cascade product $\mymacro{ \mathcal{C}} = \wfsa \circ \halfreset$, we define predicates that compute for every state $(\stateq, \statep)$, $\stateq \in \states_1$, $\statep \in \states_2$, whether $\mymacro{\mathcal{C}}$ is in it:
    \begin{equation}
        \begin{aligned}
            \mymacro{\mathcal{C}}_{(\stateq, \statep)}(\tstep) = \mathcal{B}_{\wfsa, \stateq}(\tstep) \wedge \mathcal{B}_{\halfreset, \statep}(\tstep)
        \end{aligned}
    \end{equation}
\end{proof}

\begin{restatable}{theorem}{poFSATOfuncormerTheorem}\label{thm:pofsa-to-transformer}
    Let $\wfsa$ be a \pofaAcr{}.
    Then, there exists an equivalent transformer $\tf \in \tfFL$.
\end{restatable}
\begin{proof}
    Let $\mymacro{ \mathcal{C}} = \mathcal{B}_0 \circ \cdots \circ \mathcal{B}_{\idxk}$ the semiautomaton $\wfsa$ is homomorphic to. 
    Let $\homomorphism \colon \states' \to \states$ be the homomorphism from $\mymacro{ \mathcal{C}}$ to $\wfsa$ (where $\states$ are the states of $\wfsa$, $\states'$ are the states of $\mymacro{ \mathcal{C}}$). 
    By \cref{lem:two-half-resets}, we can iteratively define formulas $\mymacro{\mathcal{C}}_{\stateq'}$ for all $\stateq' \in \states'$ that simulate the semiautomaton $\mymacro{\mathcal{C}}$. If we write instructions describing the homomorphism $\homomorphism$, we can then write formulas that yield the states traversed by $\wfsa$ as:
    \begin{equation}
        \wfsa_\stateq(\tstep) = \bigvee_{\statep \in \states', \homomorphism(\statep) = \stateq} \mymacro{\mathcal{C}}_{\statep}(\tstep)
    \end{equation}
$\wfsa_\stateq(\tstep)$, however, describes the state \textit{before reading the symbol at $\tstep$} (by strict masking). We want predicates that yield the state \textit{after reading at the symbol position $\tstep$}. We thus write:
    \begin{equation}
        \outputvec_\stateq(\tstep) = \bigvee_{\substack{\statep \in \states, \sym \in \alphabet \\ \trans(\statep, \sym)=\stateq}} \wfsa_\statep(\tstep) \land \states_\sym(\tstep)
    \end{equation}
    Finally, we denote by $\final$ the set of final states in $\wfsa$. 
    We thus define the output vector $\outputvec$ by:
    \begin{equation}
        \outputvec(\tstep) = \bigvee_{\stateq \in \final}\outputvec_\stateq(\tstep)
    \end{equation}
    $\outputvec(\length) = 1$ if and only if $\wfsa$ is in one of the final states when reading the entire string. 
    This concludes the translation of $\wfsa$ to $\braspFL$, showing the existence of a $\tf \in \tfFL$ equivalent to $\wfsa$.
\end{proof}

\section{Direct Translations of $\tfFL$} \label{sec:direct-proofs}

\subsection{A Normal Form for $\tfFL$ Transformers} \label{app:transformers-normal-form} 
The representational capacity of transformers depends tightly on the modeling assumptions. 
\cref{sec:tfFL-braspFL} studies the same architecture as \citet{yang2024maskedhardattentiontransformersrecognize}.
In this section, we outline a formalization of $\tfFL$ transformers in a form that allows us to describe direct translations to $\ptl$ and \pofaAcr{}s in \cref{sec:direct-translations} more easily.
The idea of this is similar to the normal form of \citet{hao2022formallanguagerecognitionhard}.

\subsubsection{Finite Precision and Simplifications} \label{sec:finite-precision}
As \citet{yang2024maskedhardattentiontransformersrecognize}, we work with finite-precision transformers with no positional encodings. 
Further, we focus on strict future masking and leftmost hard attention, defined formally below.

In our formalization, we omit element-wise transformations such as the query, key, and value transformations and element-wise MLPs.
Instead, we directly aggregate the original representations into (constant-size) vectors that can take increasingly many values as the number of layers increases.\footnote{This is formalized in \cref{lem:number-of-representations-per-layer}.}
The intuition and motivation for this stems from the use of hard attention:
At any layer, hard attention augments the current contextual representation with the representation of one (preceding) symbol---the $\hardmax$.
Thus, all information that the transformer can capture is already captured if the contextual representations keep around the identities of the elements that were returned by $\hardmax$.
The element-wise representations, in this case, do not provide any additional representational power.
We elaborate on this in \cref{sec:uha}.

Note that we do not use layer normalization in our transformers and focus on a single-head attention mechanism.
Both can be incorporated into our proofs analogously to the solutions presented in \citet{yang2024maskedhardattentiontransformersrecognize} (see \citet[][\S 4.1]{yang2024maskedhardattentiontransformersrecognize} for a discussion of multi-head attention and \citet[][\S 4.3]{yang2024maskedhardattentiontransformersrecognize} for layer normalization).

\subsubsection{The Attention Mechanism}
The central component of a transformer is the \defn{transformer layer}.
\begin{definition}
    A \defn{transformer layer} $\tflayer\colon \kleeneplus{\mleft(\R^\hiddDim\mright)} \to \kleeneplus{\mleft(\R^\hiddDim\mright)}$ is a length-preserving function defined as 
    \begin{subequations}
        \begin{align}
            \tflayer(\mleft(\vx_1, \ldots, \vx_\length\mright)) &\defeq \mleft(\vy_\tstep\mright)_{\tstep = 1}^\length \\
            \mleft(\vy_1, \ldots, \vy_\length\mright) &\defeq \attnFun{\mleft(\vx_1, \ldots, \vx_\length\mright)} + \mleft(\vx_1, \ldots, \vx_\length\mright) \nonumber
        \end{align}
    \end{subequations}
    We will use $\tflayerPos{\tstep} \defeq \projfunc{\tstep} \circ \tflayer$ for the function that computes $\tflayer(\mleft(\vx_1, \ldots, \vx_\length\mright))$ and extracts the contextual representation of the $\tstep\textsuperscript{th}$ symbol by projecting out that dimension.
\end{definition}
The \defn{attention mechanism} $\attn$ is specified with the following components:
\begin{itemize}[nosep, noitemsep]
    \item A \defn{scoring function} $\scoref\colon \R^\hiddDim \times \R^\hiddDim \to \R$.
    \item A \defn{masking function} $\mask\colon \N \times \N \to \set{0,1}$ that determines the positions attended to. We use future and past masking as in \brasp. 
    We write $\unmaskedSetFun{\tstep} \defeq \{\tstep' \mid \maskFun{\tstep, \tstep'} = 1\}$.
    \item A \defn{normalization function} $\normf\colon \R^\length \to \simplexFun{\length - 1}$ that normalizes the attention scores.
\end{itemize}
We then define the attention mechanism as:
\begin{subequations}
    \begin{align}
        \attnFun{\mleft(\vx_1, \ldots, \vx_\length\mright)} &\defeq \mleft(\vy_1, \ldots, \vy_\length\mright) \\
        \vy_\tstep &\defeq \sum_{\tstep' \in \unmaskedSetFun{\tstep}} \evs_{\tstep'} \vx_{\tstep'} \label{eq:attention-sum} \\
        \vs &\defeq \normf((\scoref\mleft(\vx_{\tstep'}, \vx_\tstep\mright))_{\tstep' \in \unmaskedSetFun{\tstep}}) \label{eq:attention-norm}
    \end{align}
\end{subequations}

The $\sum$-notation in \cref{eq:attention-sum} can naturally be thought of as \emph{collecting} information from the unmasked positions.
Intuitively, if the space of contextual representations is large enough, this can be interpreted as \emph{concatenating} the representations together.
This will be particularly useful in our \uha formulation, where only a single $\vx_{\tstep'}$ will be selected and the finiteness of the representation space will mean that \emph{all} relevant information about the string $\strlet$ will be stored in $\vx_\tstep$.
See also \cref{sec:uha}.

Let $\nLayers \in \NgtZero$ and $\tflayer^{(1)}, \ldots, \tflayer^{(\nLayers)}$ be transformer layers.
A \defn{transformer} $\tf\colon \kleene{\alphabet} \to (\R^{\hiddDim})^{+}$ is a composition of transformer layers:
\begin{equation} \label{eq:transformer-layer-composition}
    \tf \defeq \tflayer^{(\nLayers)} \circ \cdots \circ \tflayer^{(1)} \circ \embedFunc
\end{equation}
where $\embedFunc\colon \kleene{\alphabet} \to \kleeneplus{\mleft(\R^\hiddDim\mright)}$ is a position-wise embedding function that maps symbols to their static representations.

We also write 
\begin{equation}
    (\vx^{(\layerIdx)}_{1}, \ldots, \vx^{(\layerIdx)}_{\length}) = \mleft(\tflayer^{(\layerIdx)} \circ \cdots \circ \tflayer^{(1)} \circ \embedFunc\mright) \mleft(\str\mright)
\end{equation}
for some layer index $\layerIdx \in \NTo{\nLayers}$ and string $\str = \sym_1 \cdots \sym_\length \in \kleene{\alphabet}$.
We call $\vx^{(\layerIdx)}_{\tstep}$ the \defn{contextual representation} of symbol $\sym_\tstep$ at layer $\layerIdx$.

A transformer computes the contextual representations of the string $\str = \sym_1 \cdots \sym_{\length} \eos$ as 
\begin{equation}
    (\vx^{(\nLayers)}_{1}, \ldots, \vx^{(\nLayers)}_{\length}, \vx^{(\nLayers)}_{\eos}) \defeq \tf\mleft(\str\mright).
\end{equation}
We take $\vx^{(\layerIdx)}_{\eos}$ to be the representation of the entire string.
This motivates the definition of the \defn{transformer encoding function} $\tfencfun$:
\begin{equation} \label{eq:enc}
    \tfencfun\mleft(\str\mright) \defeq \vx_{\eos}^{(\nLayers)}.
\end{equation}
This allows us to define a transformer's language.
This is usually defined based on a linear classifier based on $\tfencfun\mleft(\str\mright)$:
\begin{equation}
    \langFun{\tf} \defeq \set{\str \in \kleene{\alphabet} \mid \vtheta^\top \tfencfun\mleft(\str\mright) > 0}
\end{equation}
for some $\vtheta \in \R^\hiddDim$.
Since we are working with finite-precision transformers, the set of possible $\tfencfun\mleft(\str\mright)$ is finite (in our normal form, it is a subset of $\alphabet^{2^{\nLayers}}$).
We can thus equate the condition $\vtheta^\top \tfencfun\mleft(\str\mright) > 0$ with $\tfencfun\mleft(\str\mright)$'s inclusion in a subset of $\alphabet^{2^{\nLayers}}$ and define the language of a transformer as follows.
\begin{definition} \label{def:transformer-language}
    Let $\tf \in \tfFL$.
    We define its language $\langFun{\tf}$ as 
    \begin{equation} \label{eq:transformer-acc}
        \langFun{\tf} \defeq \set{\str \in \kleene{\alphabet} \mid \tfencfun\mleft(\str\mright) \in \sF_\tf}
    \end{equation}
    where $\sF_\tf$ is a set of accepting final representations.
\end{definition}

\subsubsection{Unique Hard Attention} \label{sec:uha}
Let $\tiebreak \in \set{\min, \max}$ and let $\vs \in \R^N$.
We define 
\begin{equation}
    \hardmax\mleft(\vs\mright)_n \defeq \begin{cases}
        1 &\ifcondition n = \tiebreak( \argmax\mleft(\vs\mright))\\
        0 &\otherwisecondition
    \end{cases}
\end{equation}
Here, $\argmax\vs$ denotes the set of indices attaining the maximum value in the vector $\vs$. The function $\tiebreak$ selects a unique index from this set: $\tiebreak = \max$ corresponds to \defn{rightmost tiebreaking} $\rightmost$ and $\tiebreak = \min$ corresponds to \defn{leftmost tiebreaking} $\leftmost$. 
\begin{definition}  \label{def:unique-hard-attention}
    \defn{Unique hard attention} is computed with the $\hardmax$ normalization function, i.e., $\normf = \hardmax$ in \cref{eq:attention-norm}.
\end{definition}
With some abuse of notation, we sometimes write $\hardmax\mleft(\vs\mright)$ for the \emph{position} $\tiebreak( \argmax\mleft(\vs\mright))$.

We denote future or past masking with $\futuremask$ or $\pastmask$, respectively. 
$\tfFL$, for example, denotes the class of transformers with future masking and leftmost attention.

The following lemma is a restatement of \citet[][Lem. 22]{yang2024maskedhardattentiontransformersrecognize}.
\begin{restatable}{lemma}{nbRepresentationsLemma} \label{lem:number-of-representations-per-layer}
    Let \tf be a \uha transformer over $\alphabet$.
    Denote with $\vx^{(\layerIdx)}_\tstep$ the contextual representation of the symbol $\sym_\tstep$ at layer $\layerIdx$.
    The following holds:
    \begin{equation}
        |\set{\vx^{(\layerIdx)}_\tstep \mid \str \in \kleene{\alphabet}, \tstep \in \NTo{|\str|}}| \leq \nsymbols^{2^{\layerIdx}}.
    \end{equation}
\end{restatable}
\begin{proof}
    We prove the statement by induction on the number of layers.
    
    \paragraph{Base case:} $\layerIdx = 1$. 
    In the first layer, as we have static representations for symbols, the embedding at some position $\tstep$ is uniquely determined by the symbol $\sym_\tstep$ at that position. We thus have exactly $\nsymbols$ possible representations for a given position, regardless of the length of the string. 
    
    \paragraph{Inductive step:} $\layerIdx > 1$. 
    Suppose that the invariance holds for $\layerIdx - 1$: $|\set{\vx^{(\layerIdx - 1)}_\tstep \mid \str \in \kleene{\alphabet}, \tstep \in \NTo{|\str|}}| \leq \nsymbols^{2^{\layerIdx - 1}}$.
    For any position $\tstep$ in the string, at layer $\layerIdx$, we will compute $\vx^{(\layerIdx)}_\tstep = \attnFun{\mleft(\vx_1, \ldots, \vx_\length\mright)} + \mleft(\vx_1, \ldots, \vx_\length\mright)$ where $\vx^{(\layerIdx-1)}_\tstep$ is the representation of the symbol at the previous layer $\layerIdx-1$. 
    By the induction hypothesis, the element $\vx^{(\layerIdx-1)}_\tstep$ takes one out of at most $\nsymbols^{2^{\layerIdx-1}}$ possible values. 
    Moreover, the attention mechanism selects one element from the previous layer $\layerIdx-1$, which holds one out of at most $\nsymbols^{2^{\layerIdx-1}}$ possibles values by the induction hypothesis. 
    The element $\va_\tstep$ holds thus one out of $\nsymbols^{2^{\layerIdx-1}} \times \nsymbols^{2^{\layerIdx-1}} = \nsymbols^{2^{\layerIdx}}$ representations, concluding the induction step and the proof.
\end{proof}

\paragraph{Contextual representations as elements of a finite set.}
\cref{lem:number-of-representations-per-layer} allows us to simplify notation: Any representation of a symbol at layer $\layerIdx$ is uniquely identified by $2^{\layerIdx}$ symbols, i.e., $\vx^{(\layerIdx)} \in \alphabet^{2^{\layerIdx}}$.
We think of this as the collection of representations attended to at each of layer $\layerIdx' < \layerIdx$, since each selects a single position to be added to the current representation.
We will thus refer to $\vx^{(\layerIdx)}$ as elements of $\alphabet^{2^{\layerIdx}}$.

\subsubsection{An Invariance}
In this section and in \cref{subsec:proofs_tffl_as_ptl}, we use $\alphabet$ for the alphabet of input symbols and $\generalAlphabet$ for a general alphabet (finite set) of relevance.
Later, $\generalAlphabet$ will correspond to sets of the form $\alphabet^{2^\layerIdx}$ for some $\layerIdx \in \N$.

\begin{definition}
    Let $\generalAlphabet$ be an alphabet and $\str \in \kleene{\generalAlphabet}$.
    We define the \defn{symbol order} $\symOrdFun{\str}$ of $\str$ as the string obtained from $\str$ by keeping only the first occurrence of each symbol.
\end{definition}
We define the following relation on $\kleene{\generalAlphabet}$:
\begin{equation}
    \str \ordEq \str' \iff \symOrdFun{\str} = \symOrdFun{\str'}.
\end{equation}
It is not hard to verify that $\ordEq$ is an equivalence relation on $\kleene{\generalAlphabet}$ and to verify that $|\kleene{\generalAlphabet}/\ordEq| = |\orderedSubsets{\generalAlphabet}|$, where $\orderedSubsets{\generalAlphabet}$ is the (finite) set of all ordered subsets of $\generalAlphabet$.
We denote the equivalence class of $\str \in \kleene{\generalAlphabet}$ by $\eqclass{\str} \in \kleene{\generalAlphabet}/\ordEq$.
We have the following important invariance.
\begin{restatable}[Attention invariance]{lemma}{outputDependenceLemma} \label{lem:dependence}
    Let $\tflayer$ be an $\tfFL$ transformer layer over $\generalAlphabet$.
    For any $\str, \str' \in \kleene{\generalAlphabet}$, if $\str \ordEq \str'$, then $\tflayerPos{|\str|}(\str) = \tflayerPos{|\str'|}(\str')$.
\end{restatable}
\begin{proof}
    This follows directly from the definition of leftmost hard attention: Additional occurrences of symbols to the right of the first occurrence do not change the position attended to, meaning that the output at the final symbol is the same.
\end{proof}

\subsection{$\tfFL$ as $\ptl$} \label{subsec:proofs_tffl_as_ptl}

We view a transformer layer as a function that takes a contextual representation $\vx_\tstep^{(\layerIdx)}$ and returns a function that takes in an ordered set of contextual representations $\symOrdFun{\str}$ with $\str \in \kleene{\generalAlphabet}$ and returns the representation chosen by the unique hard attention mechanism.
\begin{subequations}
    \begin{align}
        \tflayer^{(\layerIdx)}&\colon \generalAlphabet \to \maps{\orderedSubsets{\generalAlphabet}}{\generalAlphabet}, \\
        \tflayer^{(\layerIdx)}&\colon \vx \mapsto \tffuncFun{\vx}.
    \end{align}
\end{subequations}
$\tffuncFun{\vx}$ is the function that takes an ordered set of contextual representations $\sX = \mleft(\vx_1, \ldots, \vx_N\mright)$ and returns the representation chosen by \uha:\footnote{We omit the dependence of $\tffunc$ on the layer index for readability.}
\begin{subequations}
    \begin{align}
        \tffuncFun{\vx}&\colon \orderedSubsets{\generalAlphabet} \to \generalAlphabet, \\
        \tffuncFun{\vx}&\colon \sX \mapsto \vx_{\hardmax((\scoref\mleft(\vx_\tstep, \vx_{\tstep'}\mright))_{\vx_{\tstep'} \in \sX})}. \label{eq:tflay-func}
    \end{align}
\end{subequations}
We also define
\begin{subequations}
    \begin{align}
        \vx'' \preceq_\vx \vx' &\iff \scoref\mleft(\vx, \vx''\mright) \leq \scoref\mleft(\vx, \vx'\mright) \\
        \vx'' \simeq_\vx \vx' &\iff \scoref\mleft(\vx, \vx''\mright) = \scoref\mleft(\vx, \vx'\mright) \\
        \vx'' \prec_\vx \vx' &\iff \vx'' \preceq_\vx \vx' \text{ and } \neg \left(\vx'' \simeq_\vx \vx'\right).
    \end{align} 
\end{subequations}

\transformerToLTLTheorem*
\begin{proof} \label{constr:transformer-to-ltl}
    We define an $\ptl$ formula representing $\tf \in \tfFL$  with the attention mechanism implemented by the function $\tffuncFun{\vx}$ (cf. \cref{eq:tflay-func}) by specifying a set of formulas representing each layer, starting with the initial one, and continuing inductively.
    At a high level, we define the formulas $\formulaAdd{\layerIdx}{\vx}{\vy}$ to mean that the contextual representation $\vy$ is added to the contextual representation $\vx$ at layer $\layerIdx$, i.e., that $\vy$ is the maximizer of the query-key score when $\vx$ is the query: $\tffuncFun{\vx}\mleft(\sX\mright) = \vy$ for $\sX$ representing the current string.
    We construct the formulas layer by layer.

    \paragraph{Base case:} $\layerIdx = 1$.
    We begin by specifying first-layer formulas, which work over $\alphabet$.
    We define for $\syma, \symb \in \alphabet$:
    \begin{equation} \label{eq:formula-base-case}
        \formulaAdd{1}{\syma}{\symb} 
        \defeq \atom_{\syma} 
        \land \underbrace{\bigwedge_{\substack{\sym \in \alphabet: \\ \symb \prec_\syma \sym}} \neg \past \atom_\sym}_{\substack{\text{No previous symbols with} \\ \text{higher scores than \symb}}}
        \land \past \Big(\underbrace{\atom_{\symb} \land \bigwedge_{\substack{\sym \in \alphabet\setminus\set{\symb}: \symb \simeq_\syma \sym}} \neg \past \atom_{\sym}}_{\substack{\text{There is $\symb$ in the past with no} \\ \text{equally-scored symbols in its past}}}\Big).
    \end{equation}
    In words, $\symb$'s value is added to $\syma$'s static embedding if 
    \begin{enumerate*}[label=\textit{(\roman*)},mode=unboxed]
        \item there are no symbols in the past of $\syma$ that have a higher score than $\symb$ and
        \item there exists a position with $\symb$ in the past such that there are no symbols with equal scores to its left (leftmost tiebreaking).
    \end{enumerate*}

    \paragraph{Inductive step:} $\layerIdx > 1$.
    Let now $\tflayer$ be the transformer layer at layer $\layerIdx > 1$ and assume we have correctly constructed $\formulaAdd{\layerIdx'}{\vx}{\vy}$ for $\layerIdx' < \layerIdx$.

    Firstly, we define the formulas $\atom_\vx^{(\layerIdx)}$ that, analogously to $\atom$, specify the presence of contextual representations for elements in $\alphabet^{2^{\layerIdx}}$. 
    Writing $\vx \in \alphabet^{2^{\layerIdx}}$ as $\vx = (\vz_0, \vz_1, \ldots, \vz_{\layerIdx-1})$ with $\vz_0 \in \alphabet$ and $\vz_{\layerIdx'} \in \alphabet^{2^{\layerIdx'}}, \layerIdx' \in \set{0, \ldots, \layerIdx-1}$, we define:
    \begin{equation}
        \atom_\vx^{(\layerIdx)} = 
        \atom_{\vz_0} \land \underbrace{
        \bigwedge_{\layerIdx' = 1}^{\layerIdx-1} \formulaAdd{\layerIdx'}{\vx_{\leq \layerIdx'-1}}{\vz_{\layerIdx'}}}_{\substack{\text{Verify correct representations} \\ \text{throughout layers are added to $\vz_0 \in \alphabet$}}}
    \end{equation}
    where $\vx_{\leq \layerIdx'} = (\vz_0, \ldots, \vz_{\layerIdx'})$.
    This checks the presence of $\vx \in \alphabet^{2^{\layerIdx}}$ as the contextual representation by asserting that the individual representations in $\vx$ at lower levels were indeed added by checking the formulas $\formulaAdd{\layerIdx'}{\vx_{\leq \layerIdx'-1}}{\vz_{\layerIdx'}}$.
    
    We now define, for $\vx, \vy \in \generalAlphabet = \alphabet^{2^{\layerIdx}}$:
    \begin{equation}
        \formulaAdd{\layerIdx}{\vx}{\vy} = \bigvee_{\ordss \in \orderedSubsets{\generalAlphabet}} \Big[ \underbrace{\actualOrder^{(\layerIdx - 1)}(\ordss)}_{\substack{\text{Identify correct ordered subset of} \\ \text{representations in the previous layer}}} \land \underbrace{\bestsymbol^{(\layerIdx - 1)}(\vx, \vy, \ordss)}_{\substack{\text{Check if $\vy$ is best representation for $\vx$} \\ \text{given the correct ordered subset $\ordss$}}}\Big]
    \end{equation}
    Intuitively, the formula iterates over all possible ordered subsets of $\generalAlphabet$, checks which one describes the string in the past, and then asserts whether $\vy$ is the best symbol to add to $\vx$ given the set of contextual representations $\ordss$.
    Here, $\actualOrder^{(\layerIdx)}$ is a formula that checks whether $\symOrdFun{\str} = \ordss$ by making sure that the string in the past follows the same order as $\ordss$:
    \begin{equation}
        \actualOrder^{(\layerIdx)}(\ordss) \defeq
        \underbrace{\left[\past (\atom^{(\layerIdx)}_{\ordssEl_1} \land \past (\atom^{(\layerIdx)}_{\ordssEl_2} \land \cdots \land (\past \atom^{(\layerIdx)}_{\ordssEl_{|\ordss|}})))\right]}_{\text{Elements of $\ordss$ are present in correct order}} \land 
        \underbrace{\bigwedge_{\substack{\ordssEl \in \generalAlphabet \setminus \ordss}} \neg \past \atom^{(\layerIdx)}_\ordssEl}_{\substack{\text{Representations not in $\ordss$} \\ \text{are not present}}}
    \end{equation}
    Analogously to \cref{eq:formula-base-case}, $\bestsymbol^{(\layerIdx)}$ checks whether $\vy$ is the best symbol to add to $\vx$ given the set of contextual representations $\ordss$ by 
    \begin{enumerate*}[label=\textit{(\roman*)},mode=unboxed]
        \item asserting $\vx$ is at the current position,
        \item there are no representations in the past of $\vx$ given the current ordered subset $\ordss$ that have a higher score than $\vy$ and
        \item there exists a position in $\ordss$ with $\vy$ in the past such that there are no representations with equal scores to its left (leftmost tiebreaking).
    \end{enumerate*}
    \begin{equation}
        \bestsymbol^{(\layerIdx)}(\vx, \vy, \ordss) \defeq 
        \atom_{\vx}^{(\layerIdx)} 
        \land \underbrace{\bigwedge_{\substack{\vz \in \ordss: \\ \vy \prec_\vx \vz}} \neg \past \atom_\vz^{(\layerIdx)}}_{\substack{\text{No previous representations} \\ \text{with higher scores than $\vy$}}}
        \land \past \Big(\underbrace{\atom_{\vy}^{(\layerIdx)} \land \bigwedge_{\substack{\vz \in \ordss\setminus\set{\vy}: \\ \vy \simeq_\vx \vz}} \neg \past \atom_{\vz}^{(\layerIdx)} }_{\substack{\text{There is $\vy$ in the past with no} \\ \text{equally-scored representations in its past}}}\Big).
    \end{equation}

    Finally, let $\tflayer$ be the final transformer layer and let $\fset \subseteq \alphabet^{2^{\layernumber}}$ be the set of representations for $\eos$ that lead to string acceptance by $\tf$. 
    The $\ptl$ formula $\tlf$ representing $\tf$ simply has to check whether the representation for $\eos$ is in $\fset$:
    \begin{equation}
        \tlf = \bigvee_{\vx \in \fset} \atom^{(\layernumber)}_\vx.
    \end{equation}
\end{proof}

\subsection{$\tfFL$ as \pofaAcr{}s} \label{app:transformer-to-pofsa}

\transformerToPOFATheorem
\begin{proof} 
    We will construct a semiautomaton $\wfsa$ that will, after reading $\strlet$, store in its state the ordered subsets $\sX^{(\layerIdx)} \in \orderedSubsets{\alphabet^{2^\layerIdx}}$ of contextual representations for all the transformer layers.
    In other words, it will hold $\nLayers$ \emph{equivalence classes} of the string $\strlet$, one for each layer.
    This state will be updated sequentially according to the self-attention mechanism implemented by the transformer.
    
    Formally, given a transformer $\tf \in \tfFL$ with the attention mechanism implemented by the function $\tffuncFun{\vx}$ (cf. \cref{eq:tflay-func}) over the alphabet $\alphabet$, we define the semiautomaton $\wfsa = \mleft(\alphabet, \states, \trans\mright)$. 
    We take the set of states $\states$ to be
    \begin{equation}
        \states \defeq \underbrace{\orderedSubsets{\alphabet} \times \cdots \times \orderedSubsets{\alphabet^{2^\nLayers}}}_{\text{Ordered sets of representations of all layers.}}.
    \end{equation}
    For clarity, we will explicitly write out the states $\stateq \in \states$ in their components:
    \begin{equation}
        \stateq = \begin{pmatrix}
            \sX^{(0)} \\
            \sX^{(1)} \\
            \vdots \\
            \sX^{(\nLayers)}
        \end{pmatrix}
    \end{equation}
    with $\sX^{(\layerIdx)} \in \orderedSubsets{\alphabet^{2^{\layerIdx}}}$ for $\layerIdx \in \set{0, \ldots, \nLayers}$.

    $\wfsa$ will update $\sX^{(\layerIdx)}$ with new occurrences of contextual representations by ``appending'' new contextual representations.
    Let us describe how the transition function $\trans$ updates the state $\stateq$ upon reading the symbol $\sym$.
    We write $\stateq$ for the source state and $\sX^{(\layerIdx)}$ for its components, and $\stateq'$ for the target state and $\sX'^{(\layerIdx)}$ for its components.
    
    We then define $\vx'^{(0)} = \sym$ to mean that the static representation of this symbol is the symbol itself.
    For $\layerIdx \geq 1$, we define
    \begin{equation} \label{eq:tf-pofa-update}
        \vx'^{(\layerIdx)} \defeq \begin{pmatrix} 
            \vx'^{(\layerIdx - 1)} \\ 
            \tffuncFun{\vx'^{(\layerIdx - 1)}}\mleft(\sX^{(\layerIdx - 1)}\mright)
        \end{pmatrix} 
    \end{equation}
    which simulates the $\layerIdx$\textsuperscript{th} layer of $\tf$ by
    \begin{enumerate*}[label=\textit{(\arabic*)},mode=unboxed]
        \item copying the symbol's representation $\vx'^{(\layerIdx - 1)}$ from the previous layer into the first component (residual connection) and
        \item computing the attended-to representation based on all the contextual representations seen so far at the previous layer ($\sX^{(\layerIdx - 1)}$) and the symbol's representation at the previous layer ($\vx'^{(\layerIdx - 1)}$).
    \end{enumerate*}
    Crucially, \cref{eq:tf-pofa-update} can be computed \emph{in advance} (knowing the scoring function $\scoref$) for any $\sX^{(\layerIdx)} \in \orderedSubsets{\alphabet^{2^\layerIdx}}$ and $\vx \in \alphabet^{2^\layerIdx}$ for all $\layerIdx \in \set{0, \ldots, \nLayers}$ due to the finiteness of all the considered states.
        
    We then update the set of observed contextual representations as 
    \begin{equation}
       \sX'^{(\layerIdx)} \defeq \sX^{(\layerIdx)} \cup \set{\vx'^{(\layerIdx)}} \label{eq:include-new-rep}
    \end{equation}
    to incorporate the information about the new contextual representations into the ordered sets of seen contextual representations at each layer $\layerIdx \in \set{0, \ldots, \nLayers}$.
    The union in \cref{eq:include-new-rep} is to be interpreted as adding an element to an ordered set.
    
    Defining
    \begin{equation}
        \stateq' = \begin{pmatrix}
            \sX'^{(0)} \\
            \sX'^{(1)} \\
            \vdots \\
            \sX'^{(\nLayers)}
        \end{pmatrix}
    \end{equation}
    and setting $\trans(\stateq, \sym) = \stateq'$, for all choices of $\stateq$ and $\sym$, we have defined the transition function $\trans$ that updates each state with the new contextual representations at each layer.

    The set of observed contextual representations $\sX$ satisfies a partial order as the concatenation of a new representation at a new position at every step implies we transition into novel states. 
    This further implies that the semiautomaton is partially ordered.
    
    We construct an \emph{automaton} from $\wfsa$ by setting the initial and final states.
    We set the initial state to be the one with empty ordered subsets of observed representations: $\sX^{(\layerIdx)} \defeq \emptyset$ for all $\layerIdx \in \set{0, \ldots, \nLayers}$.
    A subset of $\orderedSubsets{\alphabet^{2^\nLayers}}$ will lead to $\tf$ accepting a string. 
    We set the states with $\sX^{\nLayers}$ in that set to be final, leading to an equivalent automaton.
\end{proof}

\section{Duality with $\tfPR, \ftl$}\label{app:duality}
The class of transformers $\tfFL$, the main focus of the paper, has a natural symmetric characterization in $\tfPR$: 
While $\tfFL$ can only peek strictly into the past, $\tfPR$ can symmetrically only peek strictly into the future using strict past masking and rightmost \uha. 
$\tfPR$ can be informally seen as transformers in $\tfFL$ that instead read symbols \emph{right-to-left}, only considering the future when updating a symbol representation, analogously to the duality between $\gR$-trivial and $\gL$-trivial languages \citep{BRZOZOWSKI198032}.

Thus, all results for $\tfFL$ and $\ptl$ apply analogously to the symmetric case of $\tfPR$ and $\ftl$.  
$\ftl$ is the dual of $\ptl$---it only permits peeking into the future rather than the past. 
Similarly, partially ordered reverse automata (\rpofaAcr{}) are semiautomata whose reversal (automaton constructed by inverting the directionality of the transitions) are \pofaAcr{}s. 
The reverse of a \rpofaAcr{} is then homomorphic to a cascade product of half-resets. 
We thus can write the dual statement of \cref{thm:informal-2}.
\begin{restatable}{theorem}{DualityTheorem}\label{thm:duality}
    Let $\tf \in \tfPR$ be a transformer. 
    Then, there exists an equivalent formula $\tlf \in \ftl$. 
    
    Let $\tlf \in \ftl$ be a formula. 
    Then, there exists an equivalent transformer $\tf \in \tfPR$.
\end{restatable}

\end{document}